\documentclass[11pt]{article}
\usepackage{environ}
\usepackage{graphicx}
\usepackage[margin=1in]{geometry}
\usepackage{enumitem}
\usepackage{listings}
\usepackage[margin=1in]{geometry} 
\usepackage{amsmath,amsthm,amssymb}
\usepackage{algorithm}
\usepackage[export]{adjustbox}
\usepackage{hyperref}

\usepackage[noend]{algpseudocode}
\def\BState{\State\hskip-\ALG@thistlm}
\algnewcommand\algorithmicforeach{\textbf{for each}}
\algdef{S}[FOR]{ForEach}[1]{\algorithmicforeach\ #1\ \algorithmicdo}
\NewEnviron{eqns} {%
        \scalebox{1}{\parbox{0.3\linewidth}{%
        \begin{align*}
                \BOdY           
        \end{align*}
        }}
}

\newcommand{\card}{\text{card}}
\newcommand{\supp}{\text{supp}}

\newcommand{\C}{\mathbb{C}}

\newcommand{\rbr}[1]{\left( {#1} \right)}

\newcommand{\cbr}[1]{\left\{ {#1} \right\}}
\newcommand{\abr}[1]{\left\langle {#1} \right\rangle}

\newtheorem{theorem}{Theorem}
\newtheorem{thm}[theorem]{Theorem}

\newtheorem{lemma}[theorem]{Lemma}

\theoremstyle{definition}

\newtheorem{defn}[theorem]{Definition}

\theoremstyle{remark}

\title{Recovery Guarantees for Compressible Signals with Adversarial Noise}

\makeatletter
\renewcommand\@date{{%
  \vspace{-\baselineskip}%
  \large\centering
  \begin{tabular}{@{}c@{}}
    Jasjeet Dhaliwal\textsuperscript{1,2} \\
    \normalsize jasjeet\_dhaliwal@symantec.com
  \end{tabular}%
  \quad \quad
  \begin{tabular}{@{}c@{}}
    Kyle Hambrook\textsuperscript{2} \\
    \normalsize kyle.hambrook@sjsu.edu
  \end{tabular}

  \bigskip

  \textsuperscript{1}Center for Advanced Machine Learning, Symantec\par
  \textsuperscript{2}Department of Mathematics and Statistics, San Jose State University

  \bigskip
}}
\makeatother

\begin{document}
\maketitle

\section{Abstract}
\label{abstract}
We provide recovery guarantees for compressible signals that have been corrupted with noise and extend the framework introduced in \cite{bafna2018thwarting} to defend neural networks against $\ell_0$-norm, $\ell_2$-norm, and $\ell_{\infty}$-norm attacks.   Our results are general as they can be  applied to most unitary transforms used in practice and hold for $\ell_0$-norm, $\ell_2$-norm, and $\ell_\infty$-norm bounded noise. In the case of $\ell_0$-norm noise, we prove recovery guarantees for Iterative Hard Thresholding (IHT) and Basis Pursuit (BP). For $\ell_2$-norm bounded noise, we provide recovery guarantees for BP and for the case of $\ell_\infty$-norm bounded noise, we provide recovery guarantees for Dantzig Selector (DS). These guarantees theoretically bolster the defense framework introduced in \cite{bafna2018thwarting} for defending neural networks against adversarial inputs. Finally, we experimentally demonstrate the effectiveness of this defense framework against an array of $\ell_0$, $\ell_2$ and $\ell_\infty$ norm attacks.

\section{Introduction}
\label{introduction}
Signal measurements are often corrupted due to measurement errors and can even be corrupted due to adversarial noise injection. Supposing some structure on the measurement mechanism, is it possible for us to retrieve the original signal from a corrupted measurement? Indeed, it is generally possible to do so using the theory of Compressive Sensing \cite{candes2006stable} if certain constraints on the measurement mechanism and the signal hold.  In order to make the question  more concrete, let us consider the class of machine learning problems where the inputs are compressible (i.e., approximately sparse) in some domain. For instance, images and audio signals are known to be compressible in their frequency domain and machine learning algorithms have been shown to perform exceedingly well on classification tasks that take such signals as input \cite{Krizhevsky2012, sutskever2014}. However, it was found in \cite{szegedy2013intriguing} that neural networks can be easily forced into making incorrect predictions  by adding adversarial perturbations to their inputs; see also  ~\cite{Szegedy2014,Goodfellow2015,papernot2016limitations,carlini2017towards}. Further, the adversarial perturbations that led to incorrect predictions  were shown to be very small (in either $\ell_0$-norm, $\ell_2$-norm, or $\ell_\infty$-norm) and often imperceptible to human beings. For this class of machine learning tasks, we show that it is possible to approximately recover original inputs from adversarial inputs and  defend the neural network.

In this paper, we first provide recovery guarantees for compressible signals that have been corrupted by noise bounded in either $\ell_0$-norm, $\ell_2$-norm, or $\ell_\infty$-norm. Then we extend the framework introduced in \cite{bafna2018thwarting} to defend neural networks against $\ell_0$-norm, $\ell_2$-norm and $\ell_\infty$-norm attacks. In the case of $\ell_0$-norm attacks on neural networks, the adversary can perturb a bounded number of elements in the input but has no restriction on how much each element is perturbed in absolute value. In the case of $\ell_2$-norm attacks, the adversary can perturb as many elements as they choose as long as the $\ell_2$-norm of the perturbation vector is bounded. Finally, in the case $\ell_\infty$-norm attacks, the adversary is only constrained by the amount of noise added to each pixel.  Our recovery guarantees cover all three cases and provide a partial theoretical explanation for the robustness of the defense framework against adversarial inputs. Our contributions can be summarized as follows:
\begin{enumerate}
    \item We provide recovery guarantees for IHT and BP when the noise budget is bounded in $\ell_0$-norm. 
    \item We provide recovery guarantees for BP when the noise budget is bounded in the $\ell_2$-norm. 
        \item We provide recovery guarantees for DS when the noise budget is bounded in the $\ell_\infty$-norm and also introduce an additional constraint that improves reconstruction quality.
    \item We extend the framework introduced in \cite{bafna2018thwarting} to defend  neural networks against $\ell_0$-norm, $\ell_2$-norm and $\ell_\infty$-norm bounded attacks.
\end{enumerate}

The paper is organized as follows. We present the defense framework introduced in \cite{bafna2018thwarting}, which we call Compressive Recovery Defense (CRD), in Section \ref{problem_setup}. We present our main theoretical results (i.e. the recovery guarantees) in Section \ref{results} and compare these results to related work in Section \ref{related_work}. 
We establish the Restricted Isometry Property (RIP) in Section \ref{rip_section} provide the proofs of our main results  in Sections \ref{iht}, \ref{bp}, and \ref{ds}. We show that CRD can be used to defend against $\ell_0$-norm, $\ell_2$-norm, and $\ell_\infty$-norm bounded attacks in Section \ref{experiments} and  conclude the paper in Section \ref{conclusion}.

\subsection*{Notation}

Let $x$ be a vector in $\mathbb{C}^{N}$ and let $S \subseteq \cbr{1,\ldots,N}$ with $\overline{S}=\cbr{1,\ldots,N}\setminus S$. We denote by $|S|$, the cardinality of $S$, i.e. $\text{card}(S)$. The support of $x$, denoted by $\supp(x)$, is the set of indices of the non-zero entries of $x$, that is, $\supp(x) = \cbr{i \in \cbr{1,\ldots,N} : x_i \neq 0}$. The $\ell_0$-quasinorm of $x$, denoted $\|x\|_0$, is defined to be the number of non-zero entries of $x$, i.e. $\|x\|_0 = \text{card}(\supp(x))$. 
We say that $x$ is $k$-sparse if $\|x\|_0 \leq k$. We denote by $x_S$ either the sub-vector in $\C^{S}$ consisting of the entries indexed by $S$ or the vector in $\C^{N}$ that is formed by starting with $x$ and setting the entries indexed by $\overline{S}$ to zero. For example, if $x = [4,5,-9,1]$ and $S = \cbr{1,3}$, then $x_S$ is either $[4,-9]$ or $[4,0,-9,0]$. In the latter case, note $x_{\overline{S}} = x-x_S$.  It will always be clear from context, which meaning is intended. If $A \in \C^{m \times N}$ is a matrix, we denote by $A_S \in  \C^{m \times |S|}$ the column sub-matrix of $A$ consisting of the columns indexed by $S$.

We use $x_{h(k)}$ to denote a $k$-sparse vector in $\C^{N}$ consisting of the $k$ largest (in absolute value) entries of $x$ with all other entries zero. For example, if $x = [4,5,-9,1]$ then $x_{h(2)} = [0,5,-9,0]$. Note that $x_{h(k)}$ may not be uniquely defined. 
In contexts where a unique meaning for $x_{h(k)}$ is needed, we can choose $x_{h(k)}$ out of all possible candidates according to a predefined rule (such as the lexicographic order). We also define $x_{t(k)} = x - x_{h(k)}$.

Let $x  = [ x_1 \enspace x_2]^T  \in \C^{2n}$ with $x_1, x_2 \in \C^n$, then $x$ is called $(k,t)$-sparse if $x_1$ is $k$-sparse and $x_2$ is $t$-sparse. We define $x_{h(k,t)} = [ x_{1_{h(k)}} \enspace x_{2_{h(t)}} ]^T$, which is a $(k,t)$-sparse vector in $\C^{2n}$. Again, $x_{h(k,t)}$ may not be uniquely defined,  but when a unique meaning for $x_{h(k,t)}$ is needed (such as in Algorithm \ref{IHT}), we can choose $x_{h(k,t)}$ out of all possible candidates according to a predefined rule.

\section{Main Results}
\label{main_results}
In this section we outline the problem and the framework introduced in \cite{bafna2018thwarting}, state our main theorems, and compare our results to related work. 

\subsection{Compressive Recovery Defense (CRD)}
\label{problem_setup}
Consider an image classification problem where $x \in \mathbb{C}^{n}$ is the image vector (we can assume the image is of size $\sqrt{n} \times \sqrt{n}$ for instance). Then, letting $F \in \mathbb{C}^{n \times n}$ be the unitary Discrete Fourier Transform (DFT) matrix, we get the Fourier coefficients of $x$ as $\hat{x} = Fx$.  

It is well known that natural images are approximately sparse in the frequency domain and therefore we can assume that $\hat{x}$ is $k$-sparse, that is $\|\hat{x}\|_0 \leq k$. In our example of the image classification problem, this means that our machine learning classifier can accept as input the image reconstructed from $\hat{x}_{h(k)}$, and still output the correct decision. That is, the machine learning classifier can accept $F^{*}\hat{x}_{h(k)}$ as input and still output the correct decision. Now, suppose an adversary perturbs the original image $x$ with a noise vector $e$, such that we observe $y = x + e$. Noting that $y$ can also be written as $y = F^*\hat{x} + e$, we are interested in recovering an approximation $x^{\#}$ to $\hat{x}_{h(k)}$, such that when we feed  $F^{*}x^{\#}$ as input to the classifier, it can still output the correct classification decision. 

More generally, this basic framework can be used for adversarial inputs $u = v + d$, where $v$ is the original input and $d$ is the added noise vector, as long as there exists a matrix $A$ such that $u = A\hat{v} + d$, where $\hat{v}$ is approximately sparse and $\|d\|_p \leq \eta$ for some $p,\eta\geq 0$. If we can recover an approximation $v^{\#}$ to $\hat{v}$ with bounds on the recovery error, then we can use $v^{\#}$ to reconstruct an approximation $Av^{\#}$ to $v$ with controlled error.  

This general framework was proposed by \cite{bafna2018thwarting}. Moving forward, we refer to this general framework as Compressive Recovery Defense (CRD) and utilize it to defend neural networks against $\ell_0$, $\ell_2$, and $\ell_\infty$-norm attacks. As observed in \cite{bafna2018thwarting}, $x^{[0]}$ in Algorithm \ref{IHT}, can be initialized randomly to defend against a reverse-engineering attack.  In the case of Algorithm \ref{BP}, the minimization problem can be posed as a Second Order Cone Programming (SOCP) problem and it appears non-trivial to create a reverse engineering attack that will retain the adversarial noise through the recovery and reconstruction process. The same reasoning holds for Algorithm \ref{DS} which can be posed as a Linear Programming (LP) problem.


\subsection{Results}
\label{results}
 
\begin{theorem}[$\ell_0$-norm IHT]
\label{main_result_IHT}

Let $A = [F \enspace I] \in \mathbb{C}^{n \times 2n}$, where $F \in \mathbb{C}^{n \times n}$ is a unitary matrix with $|F_{ij}|^2 \leq \frac{c}{n}$ and $I \in \mathbb{C}^{n \times n}$ is the identity matrix. 
Let $y = F\hat{x} + e$, where $\hat{x}, e  \in \C^n$, and $e$ is $t$-sparse. Let $1 \leq k \leq n$ be integer and let $ x^{[T+1]} = IHT(y,A,k,t,T)$ where $x^{[T+1]} = \left[\hat{x}^{[T+1]} \enspace e^{[T+1]}\right]^T \in \C^{2n}$ with $\hat{x}^{[T+1]}, e^{[T+1]} \in \mathbb{C}^{n}$. 

Define $\rho := \sqrt{27} \sqrt{\frac{ckt}{n}},  
\quad \tau(1-\rho) := \sqrt{3} \sqrt{1 + 2\sqrt{\frac{ckt}{n}}}$. If $0 < \rho < 1$, then:
\begin{align}
\label{main_result_IHT eq1}
\| \hat{x}^{[T+1]} - \hat{x}_{h(k)} \|_2 
\leq 
\rho^{(T+1)} \sqrt{  \| \hat{x}_{h(k)} \|_2^2 + \|e\|_2^2   } 
+ \tau\| \hat{x}_{t(k)} \|_2 
\end{align}
Moreover for any $0 < \epsilon < 1$ and any 
$T \geq \left( \frac{\log(1/\epsilon) + \log(\sqrt{  \| \hat{x}_{h(k)} \|_2^2 + \|e\|_2^2   } )}{\log(1/\rho)}\right)$, we get: 
\begin{align}
\label{main_result_IHT eq6}
\| \hat{x}^{[T+1]} - \hat{x}_{h(k)} \|_2 
\leq \tau\| \hat{x}_{t(k)} \|_2  + \epsilon 
\end{align}

Now define $\rho := 2\sqrt{2}\sqrt{\frac{ckt}{n}}, \quad \tau(1-\rho) := 2$. If $0 < \rho < 1$, then:
\begin{align}
\label{main_result_IHT eq4}
\| \hat{x}^{[T+1]} - \hat{x}_{h(k)} \|_2 
&\leq 
\rho^{(T+1)}\| \hat{x}_{h(k)}\|_2+ \tau(\|\hat{x}_{t(k)}\|_2 + \|e\|_2)
\end{align}
Moreover for any $0 < \epsilon < 1$ and any 
$T \geq \left( \frac{\log(1/\epsilon) + \log(\| \hat{x}_{h(k)}\|_2)}{\log(1/\rho)}\right)$, we get: 
\begin{align}
\label{main_result_IHT eq5}
\| \hat{x}^{[T+1]} - \hat{x}_{h(k)} \|_2 &\leq \tau(\|\hat{x}_{t(k)}\|_2 + \|e\|_2) + \epsilon
\end{align}

\end{theorem}

Note that for practical applications, \eqref{main_result_IHT eq4} and \eqref{main_result_IHT eq5} provide error bounds for larger values of $k$ and $t$ than \eqref{main_result_IHT eq1} and \eqref{main_result_IHT eq6},  at the expense of the extra error term $\|e\|_2$. Further, the above results apply to unitary transformations such as the Fourier Transform, Cosine Transform, Sine Transform, Hadamard Transform, and other wavelet transforms.  Next, we consider the recovery error for $\ell_0$-norm bounded noise with BP instead of IHT. Providing bounds for BP is useful as there are cases \footnote{ As shown in Section \ref{sub_l0} and Section \ref{l0_attack}} when BP provides recovery guarantees against a larger $\ell_0$ noise budget than IHT. We note that since Algorithm \ref{BP} is not adapted to the structure of the matrix $A$ in the statement of Theorem \ref{main_result_BP_l0}, one can expect the guarantees to be weaker.
\begin{theorem}[$\ell_0$-norm BP]
\label{main_result_BP_l0} 
Let $A = [F \enspace I] \in \mathbb{C}^{n \times 2n}$, where $F \in \mathbb{C}^{n \times n}$ is a unitary matrix with $|F_{ij}|^2 \leq \frac{c}{n}$ and $I \in \mathbb{C}^{n \times n}$ is the identity matrix. 
Let  $y = F\hat{x} + e$, and let $1 \leq k,t \leq n$ be  integers. Define
\begin{align*}
\delta_{k,t} = \sqrt{\frac{ckt}{n}}, \quad   \beta = \sqrt{\frac{\max\{k,t\}c}{n}},\quad  \theta = \frac{\sqrt{k+t}}{(1- \delta_{k,t})}\beta, \quad \tau = \frac{\sqrt{1 + \delta_{k,t}}}{1 - \delta_{k,t}}
\end{align*}

If  $0 < \delta_{k,t} <1$ and $0 < \theta < 1$, then for a solution  $x^{\#} = \text{BP}(y,A,\|\hat{x}_{t(k)}\|_2)$ of Algorithm \ref{BP}, we have the error bound 
\begin{align}
\label{main_result_BP_l0_eq1}
    \|\hat{x}^{\#} - \hat{x}_{h(k)}\|_2 \leq \left(\frac{2\tau \sqrt{k+t}}{1-\theta} \left(1 + \frac{\beta}{1 - \delta_{k,t}} \right) + 2 \tau \right)\|\hat{x}_{t(k)}\|_2 
\end{align} 
where we write $x^{\#} = [\hat{x}^{\#} \enspace e^{\#} ]^T \in \C^{2n}$ with $\hat{x}^{\#}, e^{\#} \in \mathbb{C}^{n}$.

\end{theorem}

Our third result covers the case when the noise is bounded in $\ell_2$-norm. Note that the result covers all unitary matrices and removes the restriction on the magnitude of their elements. 

\begin{theorem}[$\ell_2$-norm BP]
\label{main_result_BP_l2}
Let  $F \in \mathbb{C}^{n \times n}$ be a unitary matrix and let $y = F\hat{x} + e$, where $\hat{x} \in \mathbb{C}^{n}$ is $k$-sparse and $e \in \mathbb{C}^{n}$. If $\|e\|_2 \leq \eta$, then for a solution $x^{\#} = \text{BP}(y,F,\eta)$ of Algorithm \ref{BP}, we have the error bound
\begin{align}
\label{main_result_BP_l2_eq1}
    &\| x^{\#}-\hat{x}\|_1 \leq 4\sqrt{k}\eta  \\
\label{main_result_BP_l2_eq2}
    &\| x^{\#}-\hat{x}\|_2 \leq 6\eta 
\end{align}

\end{theorem}
Finally, we provide recovery guarantees when the noise is bounded in $\ell_\infty$-norm.
\begin{theorem}[$\ell_\infty$-norm DS]
\label{main_result_BP_li}
Let  $F \in \mathbb{C}^{n \times n}$ be a unitary matrix and let $y = F\hat{x} + e$, where $\hat{x} \in \mathbb{C}^{n}$ is $k$-sparse and $e \in \mathbb{C}^{n}$. If $\|e\|_{\infty} \leq \eta_1$ and $\|F^*e\|_{\infty} \leq \eta_{2}$, then  for a solution $x^{\#} = \text{DS}(y,F,\eta_1,\eta_2)$ of Algorithm \ref{DS}, we have the error bound
\begin{align}
\label{main_result_BP_li_eq1}
    &\| x^{\#}-\hat{x}\|_1 \leq 4k\eta_2 \\
\label{main_result_BP_li_eq2}
    &\| x^{\#}-\hat{x}\|_2 \leq 6\sqrt{k}\eta_2
\end{align}
\end{theorem}
\subsection{Comparison to Related Work}
\label{related_work}

The authors of \cite{bafna2018thwarting} introduced the CRD framework which inspired this work. In fact, the main theorem (Theorem 2.2) of \cite{bafna2018thwarting} also provides an approximation error bound  for recovery via IHT. First, we note that the statement of the Theorem 2.2 in \cite{bafna2018thwarting} is missing the required hypothesis $t=O(n/k)$. This  hypothesis appears in Lemma 3.6 of \cite{bafna2018thwarting}, which is used to prove Theorem 2.2 of \cite{bafna2018thwarting}, but it appears to have been accidentally dropped from the statement of Theorem 2.2 of \cite{bafna2018thwarting}. By making the constants explicit, the proof of Lemma 3.6 of \cite{bafna2018thwarting} gives the same restricted isometry property that we do in Theorem \ref{RIPthm}. Therefore, the guarantees of \eqref{main_result_IHT eq2} and \eqref{main_result_IHT eq3} are essentially the same as those in Theorem 2.2 in \cite{bafna2018thwarting}. The main difference is that, to derive recovery guarantees for IHT from the restricted isometry property, we utilize Theorem \ref{6.18} below (which is a modified version of Theorem 6.18 of \cite{foucart2017mathematical}) while the authors of \cite{bafna2018thwarting} utilize Theorem 3.4 in \cite{bafna2018thwarting} (which is taken from \cite{baraniuk2010model}). In addition, we also provide recovery error bounds for IHT in \eqref{main_result_IHT eq4} and \eqref{main_result_IHT eq5} of Theorem \ref{main_result_IHT} that hold for larger values of $k$ and $t$ at the expense of the additional error term $\|e\|_2$.

Other works that provide guarantees include \cite{Hein17} and \cite{cisse2017parseval} where the authors frame the problem as one of regularizing the Lipschitz  constant of a network and provide a lower bound on the norm of the perturbation required to change the classifier decision. The authors of \cite{sinha2017certifying} use robust optimization to perturb the training data and provide a training procedure that updates parameters based on worst case perturbations. A similar approach to \cite{sinha2017certifying} is  \cite{wong2017provable}  in which the authors use robust optimization to provide lower bounds on the norm of adversarial perturbations on the training data. In \cite{lecuyer2018certified}, the authors use techniques from Differential Privacy \cite{dwork2014algorithmic} in order to augment the training procedure of the classifier to improve robustness to adversarial inputs. Another approach using randomization is \cite{li2018second} in which the authors add i.i.d. Gaussian noise to the input and provide guarantees of maintaining classifier predictions as long as the $\ell_2$-norm of the attack vector is bounded by a function that depends on the output of the classifier.  

Most defenses against adversarial inputs do not come with theoretical guarantees.  Instead, a large body of research has focused on finding practical ways to improve robustness to adversarial inputs by either augmenting the training data \cite{Goodfellow2015}, using adversarial inputs from various networks \cite{Tramer2017EAT}, or by reducing the dimensionality of the input \cite{Xu2017FeatureSqueezing}. For instance, \cite{madry2017towards} use robust optimization to make the network robust to worst case adversarial perturbations on the training data. However, the effectiveness of their approach is determined by the amount and quality of training data available and its similarity to the distribution of the test data. An approach similar to ours but without any theoretical guarantees is \cite{samangouei2018defense}. In this work, the authors use Generative Adversarial Networks (GANs) to estimate the distribution of the training data and during inference, use a GAN to reconstruct an input that is most similar to a given test input and is not adversarial.

\section{Restricted Isometry Property}
\label{rip_section}

We now establish the restricted isometry property for certain structured matrices. First, we give some definitions. 

\begin{defn}
Let $A$ be a matrix in $\C^{m \times N}$, let $M \subseteq \C^N$, and let $\delta \geq 0$. We say that $A$ satisfies the $M$-restricted isometry property (or M-RIP) with constant $\delta$ if 
\begin{align*}
(1-\delta)\|x\|^2_2 \leq \|Ax\|^2_2 \leq (1+\delta)\|x\|^2_2
\end{align*}
for all $x \in M$. 
 
\end{defn}

\begin{defn}
We define $M_{k}$ to be the set of all $k$-sparse vectors in $\C^N$ and similarly define $M_{k,t}$ to be the set of $(k,t)$-sparse vectors in $\C^{2n}$.  In other words, $M_{k,t}$ is the following subset of $\mathbb{C}^{2n}$: 
\begin{align}
M_{k,t} = \left\{x = [x_{1} \enspace x_2 ]^T \in \mathbb{C}^{2n}: x_1 \in \mathbb{C}^{n}, x_2 \in \mathbb{C}^{n}, \|x_1\|_0 \leq k, \|x_2\|_0 \leq t \right\} \nonumber
\end{align}
We define $S_{k,t}$ to be the following collection of subsets of $\cbr{1,\ldots,2n}$:  
$$
S_{k,t} = \cbr{S_1 \cup S_2 : S_1 \subseteq \cbr{1,\ldots,n}, S_2 \subseteq \cbr{n+1,\ldots,2n}, \text{card}(S_1) \leq k, \text{card}(S_2) \leq t  }
$$
Note that $S_{k,t}$ is the collection of supports of vectors in $M_{k,t}$. 
\end{defn}

\begin{thm}
\label{RIPthm}
Let $A = [F \enspace I] \in \mathbb{C}^{n \times 2n}$, where $F \in \mathbb{C}^{n \times n}$ is a unitary matrix with $|F_{ij}|^2 \leq \frac{c}{n}$ and $I \in \mathbb{C}^{n \times n}$ is the identity matrix. Then 
\begin{align}\label{RIPthm ineq}
\rbr{1-\sqrt{\dfrac{ckt}{n}}} \| x \|^2_2 \leq \|Ax\|^2_2 \leq \rbr{1+\sqrt{\dfrac{ckt}{n}}} \| x \|^2_2
\end{align}
for all $x \in M_{k,t}$. 
In other words, $A$ satisfies the $M_{k,t}$-RIP property with constant $\sqrt{\dfrac{ckt}{n}}$. 
\end{thm} 
\begin{proof}
In this proof, if $B$ denotes an matrix in $\C^{n \times n}$, then $\lambda_1(B),\ldots, \lambda_n(B)$ denote the eigenvalues of $B$ ordered so that  $|\lambda_1(B)| \leq \cdots \leq |\lambda_n(B)|$. 
It suffices to fix an $S = S_1 \cup S_2 \in S_{k,t}$ and prove \eqref{RIPthm ineq} for all non-zero $x \in \C^{S}$. 

Since $A^{*}_{S} A_S$ is normal, there is an orthonormal basis of eigenvectors $u_1,\ldots,u_{k+t}$ for $A_S^{*} A_S$, where $u_i$ corresponds to the eigenvalue $\lambda_i(A^{*}_{S} A_S)$. 
For any non-zero $x \in \C^{S}$, we have $x = \sum_{i=1}^{k+t} c_i u_i$ for some $c_i \in \C$, so  
\begin{align}\label{RIPthm 1}
\frac{\|Ax\|_2^2}{\|x\|_2^2} = \frac{\abr{A_S^{*} A_S x, x}}{\abr{x,x}} = \frac{\sum_{i=1}^{k+t} \lambda_i(A_S^{*}A_S) c_i^2}{\sum_{i=1}^{k+t}  c_i^2}.  
\end{align}

Thus it will suffice to prove that 
$|\lambda_i(A_S^{*}A_S)-1| \leq \sqrt{\frac{ckt}{n}}$
for all $i$. Moreover, 
\begin{align}\label{RIPthm 2}
|\lambda_i(A_S^* A_S) - 1| = |\lambda_i(A_S^* A_S - I)| = \sqrt{   \lambda_i \rbr{ (A_S^* A_S - I)^{*}(A_S^* A_S - I)     } }
\end{align}
where the last equality holds because $A_S^* A_S - I$ is normal. 
By combining \eqref{RIPthm 1} and \eqref{RIPthm 2}, we see that \eqref{RIPthm ineq} will hold upon showing that the eigenvalues
of $(A_S^* A_S - I)^{*}(A_S^* A_S - I)$ are bounded by $ckt/n$. 

So far we have not used the structure of $A$, but now we must. Observe that $(A_S^* A_S - I)^{*}(A_S^* A_S - I)$ is a block diagonal matrix with two diagonal blocks of the form $X^{*}X$ and $XX^{*}$. Therefore the three matrices $(A_S^* A_S - I)^{*}(A_S^* A_S - I)$, $X^*X$, and $XX^*$ have the same non-zero eigenvalues. Moreover, $X$ is simply the matrix $F_{S_1}$ with those rows not indexed by $S_2$ deleted. The hypotheses on $F$ imply that the entries of $X^*X$ satisfy 
$
|(X^*X)_{ij}| \leq \frac{ct}{n}.
$
So the Gershgorin disc theorem implies that each eigenvalue $\lambda$ of $X^*X$ and (hence) of  $(A_S^* A_S - I)^{*}(A_S^* A_S - I)$ satisfies 
$|\lambda| \leq \frac{ckt}{n}$.
 
\end{proof}

\section{Iterative Hard Thresholding}
\label{iht}

First we present Theorem \ref{6.18} and then use it to prove Theorem \ref{main_result_IHT}.

\begin{algorithm}
    \caption{$(k,t)$-Iterative Hard Thresholding}
    \label{IHT}
    \textbf{Input:} The observed vector $y \in \mathbb{C}^{n}$, the measurement matrix $A \in \mathbb{C}^{n \times 2n}$, and positive integers $k,t, T \in \mathbb{Z}^{+}$\\
    \textbf{Output:} $x^{[T+1]} \in M_{k,t}$
        \begin{algorithmic}[1] 
        \Procedure{IHT}{$y,A,k,t,T$} 
            \State $x^{[0]}\gets 0$
            \For{$i \in [0, \dots, T]$} 
                \State $z^{[i+1]} \gets x^{[i]} + A^*(y - Ax^{[i]})$
                \State $x^{[i+1]} = (z^{[i+1]})_{h(k,t)}$
            \EndFor
            \State \textbf{return} $x^{[T+1]}$
        \EndProcedure
    \end{algorithmic}
\end{algorithm}

\begin{theorem}
\label{6.18}

Let $A \in \mathbb{C}^{n \times 2n}$ be a matrix. Let $1 \leq k,t \leq n$ be positive integers and suppose  $\delta_3$ is a $M_{3k,3t}$-RIP constant for $A$ and that $\delta_2$ is a $M_{2k,2t}$-RIP constant for $A$. 
Let $x \in \C^{2n}$, $r \in \C^n$, $y = Ax+r$, and $S \in S_{k,t}$. Letting $x^{[T+1]} = IHT(y,A,k,t,T)$, if $\delta_3 < 1/\sqrt{3}$, then we have the approximation error bound 
\begin{align*}
\| x^{[T+1]} - x_S\|_2 \leq \rho^{(T+1)} \| x^{[0]} - x_S \|_2 + \tau \|Ax_{\overline{S}} + r\|_2 
\end{align*}

where $\rho := \sqrt{3} \delta_3 < 1$ and $(1-\rho)\tau = \sqrt{3}\sqrt{1+\delta_2} \leq 2.18$. Thus, the first term on the right goes to $0$ as $T$ goes to $\infty$.  
\end{theorem} Theorem \ref{6.18} is a modification of Theorem 6.18 of \cite{foucart2017mathematical}. More specifically, Theorem 6.18 of \cite{foucart2017mathematical} considers $M_{3k}$, $M_{2k}$, and $S_{k}$ in place of $M_{3k,3t}$ and $M_{2k,2t}$ and $S_{k,t}$ and any dimension $N$ in place of $2n$. The proofs are very similar, so we omit the proof of Theorem \ref{6.18}. We will now prove a lemma that will be required for the proof of Theorem \ref{main_result_IHT}. For the proof of Lemma \ref{partial_RIP} and Theorem \ref{main_result_IHT}, we use the following convention: let $A \in \C^{m \times N}$ be a matrix, then, we denote by $(A)_S$, the ${m \times N}$ matrix that is obtained by starting with $A$ and zeroing out the columns indexed by $\overline{S}$. Note that $(A)_{S} = A - (A)_{\overline{S}}$.

\begin{lemma}
\label{partial_RIP}
Let $F \in \C^{n \times n}$ be a unitary matrix with $|F_{ij}|^2 \leq \frac{c}{n}$ and let $S \subseteq [n]$ be a index set with $|S| = t$.  Then for any $k$-sparse vector $z \in \C^{n}$, we have:
\begin{align}
    \|(F^{*})_SFz\|^2_2 \leq \frac{ktc}{n}\|z\|_2^2 \nonumber
\end{align}
\end{lemma}

\begin{proof}
First note that  $(F^{*})_{S} 
\in \C^{n \times n}$ contains only $t$ non-zero columns since $|S| = t$ Therefore, we have $|((F^{*})_{S}F)_{ij}| \leq \frac{tc}{n}$ since $|F_{ij}|^{2} \leq \frac{c}{n}$. Further, since the non-zero columns of $(F^{*})_{S}$ are orthogonal to each other, we get $((F^{*})_{S})^{*}(F^{*})_{S} = (I)_{S}$, where $I \in \C^{n \times n}$  is the identity matrix. Using this, we have for any $w \in \C^n$,
\begin{align}
    \|(F^{*})_{S}Fw\|_2^2 = \langle (F^{*})_{S}Fw, (F^{*})_{S}Fw \rangle =  \langle ((F^{*})_{S}F)^{*} (F^{*})_{S}Fw, w \rangle = \langle (F^{*})_{S}Fw, w \rangle = |\abr{(F^{*})_{S}Fw, w }| \nonumber 
\end{align}

Now let $V \subseteq [n]$ be any index set with cardinality $k$, that is $|V| = k$ and let $z \in \C^n$ be any vector supported on $V$.  We then get, 
\begin{align}
   \|(F^{*})_{S}Fz\|_2^2 = \left |\abr{(F^{*})_{S}Fz,z} \right | = \left |\sum_{k\in V} z^{*}_{k} \left(\sum_{j \in V} ((F^{*})_{S}F)_{kj}z_{j}\right)\right | &\leq \sum_{k\in V} |z^{*}_{k}| \left(\sum_{j \in V} |((F^{*})_{S}F)_{kj}||z_{j}|\right) \nonumber \\
    &\leq \sum_{k\in V} |z^{*}_{k}| \left(\frac{tc}{n}\sum_{j \in V} |z_{j}|\right) \nonumber \\
&= \frac{tc}{n}\|z\|_1^2 \leq \frac{ktc}{n}\|z\|_2^2 \nonumber 
\end{align}
where we use the fact that $z$ is $k$-sparse for the last inequality.
\end{proof}

\noindent Now we provide the proof for Theorem \ref{main_result_IHT}. 
\begin{proof}[\textbf{Proof of Theorem \ref{main_result_IHT}}]

Theorem \ref{RIPthm} implies that the statement of Theorem \ref{6.18} holds with 

$\delta_3 = \sqrt{\frac{c\cdot 3k \cdot 3t}{n}}$ and $\delta_2 = \sqrt{\frac{c\cdot 2k \cdot 2t}{n}}$. Noting that $y =  A [\hat{x}_{h(k)} \enspace e ]^{T} + F \hat{x}_{t(k)}$, where $[\hat{x}_{h(k)} \enspace e ]^{T} \in M_{k,t}$, set $x^{[T+1]} = IHT(y,A,k,t,T)$ and apply Theorem \ref{6.18} with $x = [ \hat{x}_{h(k)} \enspace  e]^{T}$ , $r = F\hat{x}_{t(k)}$, and $S = \supp(x)$. Letting $x^{[T+1]} = [ \hat{x}^{[T+1]} \enspace e^{[T+1]}]^T $,  use the facts that $\|  \hat{x}^{[T+1]} - \hat{x}_{h(k)} \|_2 \leq \|  x^{[T+1]} - x_S \|_2$ and $\| F\hat{x}_{t(k)} \|_2 = \|\hat{x}_{t(k)} \|_2$.
That will give \eqref{main_result_IHT eq1}. Letting $T = \left( \frac{\log(1/\epsilon) + \log(\sqrt{\| \hat{x}_{h(k)} \|_2^2 + \|e\|_2^2      })}{\log(1/\rho)}\right)$,  gives $\rho^{T} \sqrt{  \| \hat{x}_{h(k)} \|_2^2 + \|e\|_2^2   }  \leq \epsilon$,  which can be substituted in \eqref{main_result_IHT eq1} to get \eqref{main_result_IHT eq6}. Noting that $||e^{[T]} - e||_2   \leq   \tau||\hat{x}_{t(k)}||_2 + \epsilon$, we can use the same reasoning as used in \cite{bafna2018thwarting} to get:

\begin{align}
\label{main_result_IHT eq2}
\| \hat{x}^{[T+1]} - \hat{x}_{h(k)} \|_{\infty} &\leq \sqrt{\frac{2ct}{n}} \left(\tau\| \hat{x}_{t(k)} \|_2  + \epsilon\right)\\
\label{main_result_IHT eq3}
\| \hat{x}^{[T+1]} - \hat{x}_{h(k)} \|_2 &\leq \sqrt{\frac{4ckt}{n}} \left(\tau\| \hat{x}_{t(k)} \|_2   + \epsilon\right)
\end{align}

\noindent which are  the essentially the same as the results of Theorem 2.2 in \cite{bafna2018thwarting}. \\

Now we prove \eqref{main_result_IHT eq4}. Let $z^{[T]}$ be as defined in Algorithm \ref{IHT} with  $z^{[T]} = [z_1^{[T]} \enspace z_2^{[T]}]^T \in \C^{2n}$ where $z_1^{[T]}, z_2^{[T]} \in \C^{n}$. Note that $\hat{x}^{[T]} = (z_1^{[T]})_{h(k)}$. Therefore, we have  $z_1^{[T]} = F^{*}(y - e^{[T-1]})$, where  $e^{[T-1]} = (y - F\hat{x}^{[T-2]})_{h(t)}$. Now let $S$ be the set of indices selected by the hard thresholding operation $h(t)$ to get $e^{[T-1]}$. Then observe that  $z_1^{[T]}= F^{*}(y - (y - F\hat{x}^{[T-2]})_{S}) $.  Next, note that $\|z_1^{[T]} - \hat{x}^{[T]}\|^2_2 \leq \|z_1^{[T]} - \hat{x}_{h(k)}\|^2_2$ as $\hat{x}^{[T]}$ is a best $k$-sparse approximation to $z_1^{[T]}$. We can thus write,
\begin{align}
    \|(z_1^{[T]} - \hat{x}_{h(k)}) -  (\hat{x}^{[T]} - \hat{x}_{h(k)})\|_2^2 &= \|z_1^{[T]} -\hat{x}_{h(k)}\|_2^2 - 2\text{Re}\langle z_1^{[T]} -\hat{x}_{h(k)}, \hat{x}^{[T]}-\hat{x}_{h(k)}\rangle + \|\hat{x}^{[T]} - \hat{x}_{h(k)}\|_2^2 \nonumber
\end{align}
Therefore, we have,
\begin{align}
\|\hat{x}^{[T]} - \hat{x}_{h(k)}\|_2^2 &\leq 2\text{Re}\langle z_1^{[T]} -\hat{x}_{h(k)}, \hat{x}^{[T]}-\hat{x}_{h(k)}\rangle \nonumber \\
&\leq 2|\langle z_1^{[T]} -\hat{x}_{h(k)}, \hat{x}^{[T]}-\hat{x}_{h(k)}\rangle| \nonumber \\
&\leq 2\|z_1^{[T]} - \hat{x}_{h(k)}\|_2 \|\hat{x}^{[T]} - \hat{x}_{h(k)}\|_2 \nonumber 
\end{align}
If $\|\hat{x}^{[T]} - \hat{x}_{h(k)}\|_2 > 0$, then $ \|\hat{x}^{[T]} - \hat{x}_{h(k)}\|_2 \leq 2\|z_1^{[T]} - \hat{x}_{h(k)}\|_2$.  Now note that,
\begin{align}
z_1^{[T]} &= \hat{x}+ F^{*}e - F^{*}(F(\hat{x}- \hat{x}^{[T-2]}) + e)_{S} \nonumber \\
    &= \hat{x} + F^{*}e - (F^{*})_S(F(\hat{x} - \hat{x}^{[T-2]}) + e) \nonumber \\
    &= \hat{x}+ (F^{*} - (F^{*})_S)e - (F^{*})_SF(\hat{x} - \hat{x}^{[T-2]}) \nonumber 
\end{align}

\noindent Using the fact that $(F^{*})_{\overline{S}} = F^{*} - (F^{*})_{S}$, we can simplify the above to get:
\begin{align}
    \|z_1^{[T]} - \hat{x}_{h(k)}\|_2 = \|(F^{*})_{\overline{S}}F\hat{x}_{t(k)} + (F^{*})_{\overline{S}}e - (F^{*})_{S}F(\hat{x}_{h(k)}- \hat{x}^{[T-2]})\|_2 \nonumber 
\end{align} Therefore,
\begin{align}
\|\hat{x}^{[T]} - \hat{x}_{h(k)}\|_2 &\leq 2\left(\|(F^{*})_{\overline{S}}F\|_{2 \to 2}\|\hat{x}_{t(k)}\|_2+ \|(F^{*})_{\overline{S}}\|_{2\to 2}\|e\|_2 + \|(F^{*})_{S}F(\hat{x}_{h(k)}- \hat{x}^{[T-2]})\|_2\right) \nonumber \\
&\leq 2\left(\|\hat{x}_{t(k)}\|_2+ \|e\|_2\right)+ 2\|(F^{*})_{S}F(\hat{x}_{h(k)}- \hat{x}^{[T-2]})\|_2 \nonumber
\end{align}
where we use  $ \|(F^{*})_{\overline{S}}\|_{2\to 2} \leq \|F^{*}\|_{2\to 2} = 1$. Now since $\hat{x}_{h(k)} - \hat{x}^{[T-2]}$ is $2k$-sparse, we can  use the result of Lemma \ref{partial_RIP} to get:

\begin{align}
\|\hat{x}^{[T]} - \hat{x}_{h(k)}\|_2 &\leq 2\left(\|\hat{x}_{t(k)}\|_2+ \|e\|_2\right)+ 2\left(\sqrt{\frac{2ktc}{n}}\right)\|\hat{x}^{[T-2]} - \hat{x}_{h(k)} \|_2  \nonumber
\end{align}

Now let $\rho =2\sqrt{2}\sqrt{\frac{ktc}{n}}, \tau(1-\rho) = 2$ and note that if $ \rho < 1$, we can use induction on $T$ to get \eqref{main_result_IHT eq4}. Then for any $0 < \epsilon < 1$ and any $T \geq \left( \frac{\log(1/\epsilon) + \log(\| \hat{x}_{h(k)}\|_2)}{\log(1/\rho)}\right)$, we have $\rho^{T}(\| \hat{x}_{h(k)}\|_2) \leq \epsilon$ which gives us \eqref{main_result_IHT eq5}.

\end{proof}

\section{Basis Pursuit}
\label{bp}

Next we introduce the Basis Pursuit algorithm and prove its recovery guarantees for $\ell_0$-norm and $\ell_2$-norm noise. We begin by stating some definitions that will be required in the proofs of the main theorems.

\begin{algorithm}
    \caption{Basis Pursuit}
    \label{BP}
     \textbf{Input:} The observed vector $y \in \mathbb{C}^{m}$, where $y = A\hat{x} + e$, the measurement matrix $A \in \mathbb{C}^{m \times N}$, and the norm of the error vector $\eta$ such that $\|e\|_2 \leq \eta$ \\
    \textbf{Output:} $x^{\#} \in \C^{N}$
    \begin{algorithmic}[1] 
        \Procedure{BP}{$y,A,\eta$} 
            \State $x^{\#} \gets  \operatorname{arg\,min}_{z \in \mathbb{C}^{N}}\|z\|_1 \enspace \text{subject to} \|Az - y\|_2 \leq \eta$
            \State \textbf{return} $x^{\#}$
        \EndProcedure
    \end{algorithmic}
\end{algorithm}

\begin{defn}
The matrix $A \in \C^{m \times N}$ satisfies the robust null space property with constants $0 < \rho < 1$, $\tau > 0$ and norm $\| \cdot \|$ if 
for every set $S \subseteq [N]$ with $\card(S) \leq s$ and 
for every $v \in \C^N$ we have 
$$
\|v_S\|_1 \leq \rho \| v_{\overline{S}} \|_1 + \tau \|A v \|
$$
\end{defn}

\begin{defn}
The matrix $A \in C^{m \times N}$ satisfies the $\ell_q$ robust null space property of order $s$ with constants $0 < \rho < 1$, $\tau > 0$ and norm $\| \cdot \|$ if for every set $S \subseteq [N]$ with $\card(S) \leq s$ and for every $v \in \C^N$ we have 
$$
\|v_S\|_q \leq \frac{1}{s^{1-1/q}} \rho \| v_{\overline{S}} \|_1 + \tau \|A v \|
$$
Note that if $q=1$ then this is simply the robust null space property. 
\end{defn}

The proof of Theorem \ref{main_result_BP_l0} requires the following theorem (whose full proof is given in the \cite{foucart2017mathematical}). 

\begin{theorem}[Theorem 4.33 in \cite{foucart2017mathematical}]
\label{four_three_three}
Let $a_1, \dots, a_N$ be the columns of $A \in \mathbb{C}^{m \times N}$, let $x\in \mathbb{C}^{N}$ with $s$ largest absolute entries supported on $S$, and let $y = Ax + e$ with $\|e\|_2 \leq \eta$. For $\delta, \beta, \gamma, \theta, \tau \geq 0$ with $\delta < 1$, assume that: 
\begin{align*}
\|A^{*}_SA_S - I\|_{2 \to 2} \leq \delta, \quad \max_{l \in \overline{S}}\|A^{*}_Sa_l\|_2 \leq \beta,
\end{align*}
and that there exists a vector $u = A^{*}h \in \mathbb{C}^{N}$ with $h \in \mathbb{C}^{m}$ such that
\begin{align*}
    \|u_S - \text{sgn}(x_S)\|_2 \leq \gamma, \quad \|u_{\overline{S}}\|_{\infty} \leq \theta, \quad \text{and } \|h\|_2 \leq \tau \sqrt{s}.
\end{align*}

If $\rho:= \theta + \frac{\beta \gamma}{(1 - \delta)} < 1$, then a minimizer $x^{\#}$ of $\|z\|_1$ subject to $\|Az - y \|_2 \leq \eta$ satisfies:
\begin{align*}
    \|x^{\#} - x\|_2 \leq \frac{2}{(1-\rho)}\left(1 + \frac{\beta}{(1-\delta)} \right )\|x_{\overline{S}}\|_1 + \left(\frac{2(\mu\gamma + \tau\sqrt{s})}{1-\rho}\left( 1 + \frac{\beta}{1-\delta} \right)+ 2\mu \right)\eta
\end{align*}

where $ \mu := \frac{\sqrt{ 1 + \delta}}{1-\delta}$ and $\text{sgn}(x)_i =  \begin{cases} 
      0, & x_i = 0 \\
      1, &  x_i > 0 \\
      -1. &  x_i < 0  
   \end{cases}
$.
\end{theorem}
\begin{lemma}
\label{a_minus_i_norm}
Let $A \in \mathbb{C}^{n \times 2n}$, if $\|Ax\|^2_2 \leq (1 + \delta)\|x\|^2_2$ for all $x \in M_{k,t}$, then, $\|A^{*}_SA_S - I\|_{2\to 2}\leq  \delta$, for any $S \in S_{k,t}$.
\end{lemma}

\begin{proof}
Let $S \in S_{k,t}$ be given. Then for any $x \in \C^S$, we have
\begin{align}
\|A_Sx\|_2^2 - \|x\|_2^2 \leq \delta \|x\|_2^2 \nonumber 
\end{align}
 We can re-write this as : $\|A_Sx\|_2^2 - \|x\|_2^2 = \langle A_Sx, A_Sx\rangle - \langle x, x \rangle = \langle (A^{*}_SA_S - I)x, x \rangle $. Noting that $A^{*}_SA_S - I$ is Hermitian, we have:
\begin{align}
\|A^{*}_SA_S - I\|_{2\to 2} &= \max_{x \in \mathbb{C}^S \setminus \{0\}} \frac{\langle (A^{*}_SA_S - I)x, x \rangle}{\|x\|^2_2} \leq \delta \nonumber
\end{align}

\end{proof}

\begin{proof}[\textbf{Proof of Theorem \ref{main_result_BP_l0}}]
We will derive  \eqref{main_result_BP_l0_eq1} by showing that the matrix A satisfies all the hypotheses in Theorem \ref{four_three_three} for every vector in $M_{k,t}$. 

First note that by Theorem \ref{RIPthm},  $A$ satisfies the $M_{k,t}$-$RIP$ property with constant $\delta_{k,t} := \sqrt{\frac{ckt}{n}}$. Therefore, by Lemma \ref{a_minus_i_norm},  for any $S \in S_{k,t}$, we have $\|A^{*}_{S}A_{S}-I\|_{2 \to 2} \leq \delta_{k,t}$. Since $A^{*}_SA_S$ is a positive semi-definite matrix, it has only non-negative eigenvalues that lie in the range $[1 - \delta_{k,t}, 1 + \delta_{k,t}]$. Since $\delta_{k,t} < 1$ by assumption, $A^{*}_SA_S$ is injective. Thus, we can set: $h = A_S(A^{*}_SA_S)^{-1}\text{sgn}(x_S)$ and get:
\begin{align*}
    \|h\|_2 = \|A_S(A^{*}_SA_S)^{-1}\text{sgn}(x_S)\|_2\leq \|A_S\|_{2 \to 2}\|(A^{*}_SA_S)^{-1}\|_{2 \to 2}\|\text{sgn}(x_S)\|_2 \leq \tau\sqrt{k + t}
\end{align*}

where $\tau = \frac{\sqrt{1 + \delta_{k,t}}}{1 - \delta_{k,t}}$  and we  have used the following facts: since $\|A^{*}_{S}A_{S}-I\|_{2 \to 2} \leq \delta_{k,t}< 1$, we get that $\|(A^{*}_SA_S)^{-1}\|_{2 \to 2}\leq \frac{1}{1 - \delta_{k,t}}$ and that the largest singular value of $A_S$ is less than $\sqrt{1 + \delta_{k,t}}$. Now let $u = A^{*}h$, then $\|u_{S} - \text{sgn}(x_S)\|_2 = 0$. Now we need to bound the value $\|u_{\overline{S}}\|_{\infty}$. Denoting row $j$ of $A^{*}_{\overline{S}}A_S$ by the vector $v_j$, we see that it has at most $\max\{k,t\}$ non-zero entries and that $|(v_j)_l|^2 \leq  \frac{c}{n}$ for $l = 1, \dots, (k+t)$. Therefore, for any element $(u_{\overline{S}})_j$, we have:
\begin{align*}
|(u_{\overline{S}})_j| = |\langle (A^{*}_SA_S)^{-1}\text{sgn}(x_S), (v_j)^{*} \rangle | \leq \|(A^{*}_SA_S)^{-1}\|_{2 \to 2}\|\text{sgn}(x_S)\|_2 \|v_j\|_2 \leq \frac{\sqrt{k+t}}{1 -\delta_{k,t}}\sqrt{\frac{\max\{k,t\}c}{n}}
\end{align*}
Defining $\beta : =\sqrt{\frac{\max\{k,t\}c}{n}}$ and $\theta := \frac{\sqrt{k+t}}{1 - \delta_{k,t}}\beta$, we get $\|u_{\overline{S}}\|_{\infty} \leq \theta < 1$ and also  observe that  $ \max_{l \in \overline{S}}\|A^{*}_Sa_l\|_2 \leq \beta$. Therefore, all the hypotheses of Theorem \ref{four_three_three} have been satisfied. Note that $y = F\hat{x} + e = A[\hat{x}_{h(k)}\enspace e]^T + F\hat{x}_{t(k)}$, Therefore, setting  $x^{\#} = \text{BP}(y,A,\|\hat{x}_{t(k)}\|_2)$, we use the fact $\|F\hat{x}_{t(k)}\|_2 = \|\hat{x}_{t(k)}\|_2$ combined with the bound in Theorem \ref{four_three_three} to get \eqref{main_result_BP_l0_eq1}:
\begin{align}
    \|\hat{x}^{\#} - \hat{x}_{h(k)}\|_2 \leq \left(\frac{2\tau \sqrt{k+t}}{1-\theta} \left(1 + \frac{\beta}{1 - \delta_{k,t}} \right) + 2 \tau \right)\|\hat{x}_{t(k)}\|_2 \nonumber 
\end{align} 

where we write $x^{\#} = [\hat{x}^{\#} \enspace e^{\#} ]^T$ with $\hat{x}^{\#}, e^{\#} \in \mathbb{C}^{n}$.
\end{proof}

 We now focus on proving Theorem \ref{main_result_BP_l2}. In order to do so, we will need some lemmas that will be used in the main proof. 

\begin{lemma}
\label{l2impliesl1forS}
If a matrix $A \in \mathbb{C}^{m \times N}$ satisfies the $\ell_2$ robust null space property for $S \subset [N|$, with card$(S)=s$, then it satisfies the $\ell_1$ robust null space property for $S$ with constants $0<\rho<1, \tau':=\tau \sqrt{s} > 0$.
\end{lemma}

\begin{proof}
For any $v \in \mathbb{C}^{N}$, $\|v_{S}\|_2 \leq \frac{\rho}{\sqrt{s}}\|v_{\bar{S}}\|_1 + \tau \|Av\|$. Then, using the fact that $\|v_{S}\|_1 \leq \sqrt{s}\|v_{S}\|_2$, we get:$\|v_{S}\|_1 \leq \rho\|v_{\bar{S}}\|_1 + \tau \sqrt{s} \|Av\|$.

\end{proof}

\begin{lemma}[Theorem 4.20 in \cite{foucart2017mathematical}]
\label{l1impliesdiff}
If a matrix $A \in \mathbb{C}^{m \times N}$ satisfies the $\ell_1$ robust null space property (with respect to $\|.\|$) and for $0 < \rho < 1$ and $\tau > 0$ for $S \subset [N|$, then:
\begin{align}
    \|z - x\|_1 \leq \frac{1 + \rho}{1 - \rho} (\|z\|_1 - \|x\|_1 + 2\|x_{\bar{S}}\|_1) + \frac{2\tau}{1 - \rho}\|A(z -x)\| \nonumber
\end{align}
for all $z,x \in \mathbb{C}^{N}$.
\end{lemma}

\begin{lemma}[Proposition 2.3 in \cite{foucart2017mathematical}]
\label{lperror}
For any $p > q > 0$ and $x \in \mathbb{C}^{n}$, 
\begin{align}
     \inf_{z \in M_{k}} \|x-z\|_p \leq \frac{1}{(k)^{\frac{1}{q} - \frac{1}{p}}}\|x\|_q \nonumber 
\end{align}
\end{lemma}

\begin{proof}[\textbf{Proof of Theorem \ref{main_result_BP_l2}}]
Let $0 < \rho < 1$ be arbitrary. 
Since $F$ is a unitary matrix, for any $S \subseteq [n]$ and $v \in \mathbb{C}^{n}$,  we have 
\begin{align}
    \|v_S\|_2 &\leq \frac{\rho}{\sqrt{k}}\|v_{\overline{S}}\|_1 + \tau\|v\|_2 = \frac{\rho}{\sqrt{k}}\|v_{\overline{S}}\|_1 + \tau\|Fv\|_2 \label{eq_l2_ineq}
\end{align}
where $\tau = 1$. Therefore, $F$ satisfies the $\ell_2$ robust null space property for all $S \subseteq [n]$ with card$(S) \leq k$. Next, using Lemma \ref{l2impliesl1forS} we get  $\|v_{S}\|_1 \leq \rho\|v_{\bar{S}}\|_1 + \tau \sqrt{k} \|Fv\|_2$ for all $v \in \mathbb{C}^{n}$. Now let  $x^{\#} = \text{BP}(y,F,\eta)$, then we know $\|x^{\#}\|_1 \leq \|x\|_1$, where $x$ is $k$-sparse. Then letting $S \subseteq [n]$ be the support of $x$ and using the fact that $\|x_{\overline{S}}\|_2 = 0$ and Lemma  \ref{l1impliesdiff} , we get:

\begin{align}
    \|x^{\#} - x\|_1 &\leq \frac{1 + \rho}{1 - \rho} (\|x^{\#}\|_1 - \|x\|_1 + 2\|x_{\bar{S}}\|_1) + \frac{2\tau\sqrt{k}}{1 - \rho}\|F(x^{\#} -x)\|_2\nonumber \\
    &\leq  \frac{2\tau\sqrt{k}}{1 - \rho}\|F(x^{\#} -x)\|_2 \nonumber \leq \frac{4\tau\sqrt{k}}{1-\rho}\|e\|_2 \nonumber \leq \frac{4\tau\sqrt{k}}{1-\rho}\eta \nonumber
\end{align}
Letting $\rho \to 0$ and recalling that $\tau=1$ gives \eqref{main_result_BP_l2_eq1}. Now let $S$ be the support of the $k$ largest entries in $x^{\#} - x$. 
Note $\|(x^{\#} - x)_{\overline{S}}\|_2 = \inf_{z \in M_{k}} \|(x^{\#} - x)-z\|_2$. 
Then, using Lemma \ref{lperror} and \eqref{eq_l2_ineq}, we see that 
\begin{align}
        \|x^{\#} - x\|_2 &\leq \|(x^{\#} - x)_{\overline{S}}\|_2 + \|(x^{\#} - x)_{S}\|_2 \nonumber \\
        &\leq \frac{1}{\sqrt{k}}\|(x^{\#} - x)\|_1 +  \frac{\rho}{\sqrt{k}}\|(x^{\#} - x)_{\overline{S}}\|_1 + \tau\|F(x^{\#} - x)\|_2 \nonumber\\
        &\leq \frac{1 + \rho}{\sqrt{k}}\|(x^{\#} - x)\|_1 + 2\tau\eta \leq \frac{4\tau(1+\rho)}{(1-\rho)}\eta + 2\tau\eta = \left(\frac{4\tau(1+\rho)}{(1-\rho)}+ 2\tau\right)\eta \nonumber
\end{align}
Recalling $\tau = 1$ and letting $\rho \to 0$ gives the desired result. 
\end{proof}

\section{Dantzig Selector}
\label{ds}

Next we introduce the Dantzig Selector algorithm with an additional constraint. We first prove its recovery guarantees for $\ell_\infty$-norm and then explain the reasoning behind the additional constraint.

\begin{algorithm}
    \caption{Modified Dantzig Selector}
    \label{DS}
     \textbf{Input:} The observed vector $y \in \mathbb{C}^{m}$, where $y = A\hat{x} + e$, the measurement matrix $A \in \mathbb{C}^{m \times N}$, and constants $\eta_1, \eta_2$ such that $\|e\|_\infty \leq \eta_1, \|A^*e\|_\infty \leq \eta_2$ \\
    \textbf{Output:} $x^{\#} \in \C^{N}$
    \begin{algorithmic}[1] 
        \Procedure{DS}{$y,A,\eta_1, \eta_2$} 
            \State $x^{\#} \gets  \operatorname{arg\,min}_{z \in \mathbb{C}^{N}}\|z\|_1 \enspace \text{subject to} \enspace  \|A^*(y - Az)\|_\infty \leq \eta_2, \quad \|Az - y\|_\infty \leq \eta_1$
            \State \textbf{return} $x^{\#}$
        \EndProcedure
    \end{algorithmic}
\end{algorithm}

\begin{proof}[\textbf{Proof of Theorem \ref{main_result_BP_li}}]
The proof follows the same structure as the proof of Theorem \ref{main_result_BP_l2}. Therefore we provide a sketch and leave out the complete derivation. Let $0 < \rho < 1$ be arbitrary. 
Since $F$ is a unitary matrix, for any $S \subseteq [n]$ and $v \in \mathbb{C}^{n}$,  we have 
\begin{align}
    \|v_S\|_2 &\leq \frac{\rho}{\sqrt{k}}\|v_{\overline{S}}\|_1 + \|v_S\|_2 \leq \frac{\rho}{\sqrt{k}}\|v_{\overline{S}}\|_1 + \sqrt{k}\|v\|_\infty  = \frac{\rho}{\sqrt{k}}\|v_{\overline{S}}\|_1 + \sqrt{k}\|F^{*}Fv\|_\infty \nonumber
\end{align}
The rest of the argument is the same as in the proof of Theorem \ref{main_result_BP_l2}.
\end{proof}

\subsection{Additional Constraint on DS}
\label{ds_constraint}
We now comment on the additional constraint $\|Az - y\|_\infty \leq \eta_1$ in the statement of Algorithm \ref{DS}. Observe that the constraint is: $|(Az)_i - (A\hat{x})_i - e_i| \leq \eta_1$ for any $i \in [n]$. Since $\|e\|_\infty \leq \eta _1$, the constraint implies $|(Az)_i - (A\hat{x})_i| \leq 2 \eta_1$. Now suppose that our input is an image $x = F^{*}\hat{x}$,  where $F$ is the Discrete Fourier Transform matrix. Setting $A = F^{*}$,  the constraint ensures that the entries of a reconstructed image $F^{*}z$ and the original image $F^{*}\hat{x}$ will be close. That is, the constraint limits the feasible set of solutions to those that will lead to better reconstruction quality. This is particularly useful when we need to utilize the reconstructed images for some downstream task such as classification by a neural network. On the other hand suppose that we do not add this constraint. Then, since $\|F^{*}e\|_\infty \leq \sqrt{n}\eta_1$ grows with the dimension of the ambient space, there is no guarantee that the reconstruction will be close to the original image, especially when $\sqrt{n}\eta_1$ is large.

\section{Experiments}
\label{experiments}
We first analyze how our recovery guarantees perform in practice (Section \ref{recovery_error}) and then show  that  CRD can be used to defend neural networks against $\ell_0$-norm attacks (Section \ref{l0_exp}), $\ell_2$-norm attacks (Section \ref{sub_l2}) as well as $\ell_\infty$-norm attacks (Section \ref{sub_li}). 

 All of our experiments are conducted on  CIFAR-10 \cite{krizhevsky2009learning}, MNIST \cite{lecun1998mnist}, and  Fashion-MNIST \cite{xiao2017fashion} datasets with pixel values of each image normalized to lie in $[0,1]$. For every experiment, we use the Discrete Cosine Transform (DCT) and the Inverse Discrete Cosine Transform (IDCT)  denoted by the matrices $F \in \mathbb{R}^{n \times n}$ and $F^T \in \mathbb{R}^{n \times n}$ respectively. That is, for an adversarial image $y \in \mathbb{R}^{\sqrt{n} \times \sqrt{n}}$, such that, $y = x + e$, we let $\hat{x} = Fx$, and $x =F^{T}\hat{x}$,  where $x,\hat{x} \in \mathbb{R}^{n}$ and $e \in \mathbb{R}^{n}$ is the noise vector. For an adversarial image $y \in \mathbb{R}^{\sqrt{n} \times \sqrt{n} \times c}$, that contains $c$ channels, we perform recovery on each channel independently by considering $y_m = x_m + e_m$, where  $\hat{x}_{m} = F x_{m}, x_m = F^{T}\hat{x}_m$ for $m = 1,\dots, c$. The value $k$ denotes the number of largest (in absolute value) DCT co-efficients used for reconstruction of each channel, and the value $t$ denotes the $\ell_0$ noise budget for each channel.

\begin{table*}[t]
\begin{center}
\begin{tabular}{|c|c|c|}
\hline
{\bf Layer} & {\bf Type}  & {\bf Properties} \\
\hline
 1 & Convolution & 32 channels, $3 \times 3$ Kernel, No padding \\
 \hline
 2 & Convolution & 64 channels, $3 \times 3$ Kernel, No padding, Dropout with $p=0.5$ \\
 \hline
 3 & Max-pooling &  $2 \times 2$, Dropout with $p=0.5$ \\
 \hline
 4 & Fully Connected & $128$ neurons, Dropout with $p=0.5$\\
 \hline
 5 & Fully Connected & $10$ neurons \\
\hline
\end{tabular}
\end{center}
\caption{Network architecture used for MNIST and Fashion-MNIST datasets in Section \ref{l0_exp} and Section \ref{l2_exp}. The first four layers use ReLU activations while the last layer uses a softmax activation.}
\label{tab:net_arch}
\end{table*}

\begin{table*}[t]
\begin{center}
\begin{tabular}{|c|c|c|c|c|c|}
\hline
{\bf Dataset} & {\bf $t_{\text{avg}}$}  & {\bf $\delta_{\ell_{\infty}}$} & {\bf $\delta_{\ell_{2}}$} & {\bf $\Delta_{\ell_{\infty}}$} & {\bf $\Delta_{\ell_{2}}$}  \\
\hline
 CIFAR-10 & 1.52 & 0.19 & 0.27 & 5.38 & 17.34\\
 \hline
 MNIST & 1.47 & 0.14 & 0.19 & 3.46 & 9.99 \\
 \hline
 Fashion-MNIST & 1.50 & 0.10 & 0.14 & 2.87 & 8.27\\
\hline
\end{tabular}
\end{center}
\caption{Recovery performance of Algorithm \ref{IHT} on $\ell_0$-norm bounded noise.}
\label{tab:l0_recovery_iht}
\end{table*}

\begin{table*}[t]
\begin{center}
\begin{tabular}{|c|c|c|c|}
\hline
{\bf Dataset} & {\bf $t_{\text{avg}}$}  & {\bf $\delta_{\ell_{2}}$} &  {\bf $\Delta_{\ell_{2}}$}  \\
\hline
 CIFAR-10 & 4.03 & 20.10 & 866.07  \\
 \hline
 MNIST & 4.01 & 5.08 & 381.33\\
 \hline
 Fashion-MNIST & 4.02 & 6.84 & 298.73\\
\hline
\end{tabular}
\end{center}
\caption{Recovery performance of Algorithm \ref{BP} for $\ell_0$-norm bounded noise.}
\label{tab:l0_recovery_bp}
\end{table*}

We now outline the neural network architectures used for experiments in Section \ref{l0_exp} and \ref{l2_exp}. For CIFAR-10, we use the network architecture of \cite{he2016deep} while the network architecture for MNIST and Fashion-MNIST datasets is provided in Table \ref{tab:net_arch}. We train our networks using the Adam optimizer for CIFAR-10 and the AdaDelta optimizer for MNIST and Fashion-MNIST. In both cases, we use a cross-entropy loss function. We implement the following training procedure: for every training image $x$, we first generate $\hat{x}_{h(k)} = (Fx)_{h(k)}$, and then reconstruct the image $x' = F^{T}\hat{x}_{h(k)}$. We then use both $x$ and $x'$ to train the network. For instance, in MNIST we get $60000$ original training images and $60000$ reconstructed training images, for a total of $120000$ training images. The code to reproduce our experiments is available here: \url{https://github.com/jasjeetIM/recovering_compressible_signals}.

\begin{table*}[t]
\begin{center}
\begin{tabular}{|c|c|c|c|c|c|}
\hline
{\bf Dataset} & {\bf $\ell_{2_{\text{avg}}}$}  & {\bf $\delta_{\ell_{1}}$} & {\bf $\delta_{\ell_{2}}$} & {\bf $\Delta_{\ell_{1}}$} & {\bf $\Delta_{\ell_{2}}$}  \\
\hline
 CIFAR-10 & 32.03 & 88.56 & 17.45 & 2015.47 & 315.22 \\
 \hline
 MNIST & 16.15 & 43.87 & 8.20 & 365.17 & 88.80 \\
 \hline
 Fashion-MNIST & 16.16 & 40.90 & 9.32 & 367.94 & 87.64\\
\hline
\end{tabular}
\end{center}
\caption{Recovery performance of Algorithm \ref{BP} for $\ell_2$-norm bounded noise.}
\label{tab:l2_recovery_bp}
\end{table*}

\begin{table*}[t]
\begin{center}
\begin{tabular}{|c|c|c|c|c|c|}
\hline
{\bf Dataset} & {\bf $\ell_{\infty_{\text{avg}}}$}  & {\bf $\delta_{\ell_{1}}$} & {\bf $\delta_{\ell_{2}}$} & {\bf $\Delta_{\ell_{1}}$} & {\bf $\Delta_{\ell_{2}}$}  \\
\hline
 CIFAR-10 & 0.99 &  10821.49& 1435.45 & 4538.50 &  2207.48 \\
 \hline
 MNIST & 0.99 & 1519.27 & 213.23 & 2960.72 & 849.28 \\
 \hline
 Fashion-MNIST & 0.99 & 1824.30 & 274.21 & 2655.69 & 788.30\\
\hline
\end{tabular}
\end{center}
\caption{Recovery performance of Algorithm \ref{DS} for $\ell_\infty$-norm bounded noise.}
\label{tab:li_recovery_ds}
\end{table*}

\begin{figure*}[thb]
\begin{center}
\begin{tabular}{c c}
Original & {\includegraphics [valign=m,width = 4.5in, height=.3in]{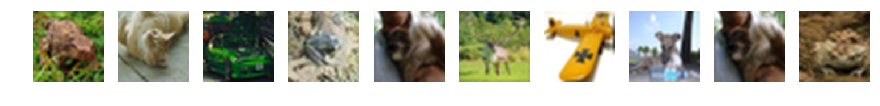}} \\
OPA & {\includegraphics[valign=m,width = 4.5in, height=.3in]{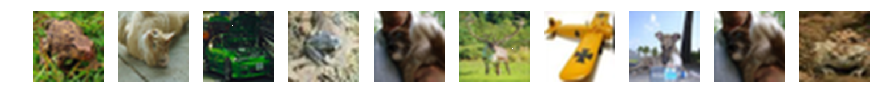} }
\end{tabular}
\end{center}
\caption{Adversarial images for CIFAR-10 created using OPA. The original images are shown in the first row, and the adversarial images are shown in the second row. }
\label{fig:adv_l0_op}
\end{figure*}

\subsection{Recovery Error}
\label{recovery_error}
\begin{table*}[t]
\begin{center}
\begin{tabular}{|c|c|c|}
\hline
 {\bf Orig. Acc.}  & {\bf OPA. Acc } & {\bf IHT. Acc.}  \\
\hline
 77.4\% & 0.0\% & 71.8\%  \\
 \hline
\end{tabular}
\end{center}
\caption{Effectiveness of CRD against OPA. The first column lists the accuracy of the network on original images and the OPA Acc. columns shows the network's accuracy on adversarial images. The  IHT. Acc. column shows the accuracy of the network on images reconstructed using Algorithm \ref{IHT}.}
\label{tab:op_recovery}
\end{table*}
Since recovery guarantees have been proved theoretically, our aim is to examine how close the recovery error is to the upper bound in practice. Each experiment is conducted on a subset of  $500$ data points sampled uniformly at random from the respective dataset. The first metric we report is $\delta_{\ell_{p}}:= \frac{1}{500}\sum_{i=1}^{500}\|x^{\#}_{i} - (\hat{x}_{i})_{h(k)}\|_p$, where $x^{\#}_{i}$  is the recovered vector for the noisy measurement $y_i$, $(\hat{x}_{i})_{h(k)} = (Fx_i)_{h(k)}$ and the average is taken over the $500$ points sampled from the dataset. This measures the average magnitude of the recovery error for the respective algorithm  in $\ell_p$ norm. In order to relate this value to the upper bound on the recovery error, we also report $\Delta_{\ell_{p}}:= \frac{1}{500}\sum_{i=1}^{500}(\Upsilon_i - \|x^{\#}_{i} - (\hat{x}_{i})_{h(k)}\|_p)$, where $\Upsilon_i$ is the guaranteed upper bound  for $y_i$. Using $\delta_{\ell_{p}}$ and $\Delta_{\ell_{p}}$, we aim to capture how much smaller the recovery error is than the upper bound for these datasets.

\subsubsection{$\ell_0$ noise}
\label{sub_l0}
\begin{figure*}[thb]
\begin{center}
\begin{tabular}{c c}
Original & {\includegraphics [valign=m,width = 4.5in, height=.3in]{images/reg-cifar10_big-op.png}} \\
OPA-Rec. & {\includegraphics[valign=m,width = 4.5in, height=.3in]{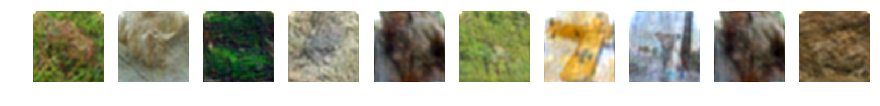} }
\end{tabular}
\end{center}
\caption{Reconstruction quality of images using Algorithm \ref{IHT}. The first row shows the original images while the second row shows reconstruction from the largest $275$ DCT co-efficients recovered using Algorithm \ref{IHT}.   }
\label{fig:adv_l0_op_recon}
\end{figure*}
For each data point $x_i \in \mathbb{R}^{n}, i = 1,2, \dots, 500$, we construct a noise vector $e_i \in \mathbb{R}^{n}$ as follows: we first sample an integer $t_i$ from a uniform distribution over the set $\{1,\dots,t\}$, where $t$ is the allowed $\ell_0$ noise budget. Next, we select an index set $S_{t_{i}} \subset [n]$ uniformly at random, such that card($S_{t_{i}}) = t_i$. Then for each $j \in S_{t_{i}}$, we set $(e_{i})_{j} = c_j$, where $c_j$ is sampled from the uniform distribution on $[0,1)$ and  $(e_i)_l = 0$ for  $l \notin S_{t_i}$ We then set $y_i = x_i + e_i$ as the observed noisy vector.  We report $\delta_{\ell_{\infty}}, \delta_{\ell_2}, \Delta_{\ell_\infty}, \Delta_{l_{2}}$ and $t_{\text{avg}} := \frac{1}{500}\sum_{i=1}^{500}t_i$.\\

\noindent \textbf{Recovery with Algorithm \ref{IHT}}\\
We examine how \eqref{main_result_IHT eq2} and \eqref{main_result_IHT eq3} perform in practice. To do so, we set $k=4$ for MNIST and Fashion-MNIST and are allowed an $\ell_0$ noise budget of $t=3$. For CIFAR-10, we set $k=5$ and are allowed a noise budget of $t=3$.  We note that $k$ values have been chosen to meet our computational constraints. As such, any other values that fit the hypotheses  of \eqref{main_result_IHT eq2} and \eqref{main_result_IHT eq3} would work just as well. The results in Table \ref{tab:l0_recovery_iht} show that on average, the recovery error is well below the upper bounds.

\noindent \textbf{Recovery with Algorithm \ref{BP}}\\
We implement Algorithm \ref{BP} using the open source library CVXPY \cite{cvxpy}. We set $k=8$ for MNIST and Fashion-MNIST and are allowed an $\ell_0$ noise budget of $t=8$. For CIFAR-10, we set $k=10$ and are allowed a noise budget of $t=8$. We observe the results in Table \ref{tab:l0_recovery_bp} and note once again that the recovery error is well below the upper bound of \eqref{main_result_BP_l0_eq1}.

\subsubsection{$\ell_2$ noise}
\label{sub_l2}
 Now  we consider the case when the noise vector $e_i, i = 1,2, \dots 500$ is only bounded in $\ell_2$-norm.  For each $e_i, i = 1,2 \dots 500$, we set $(e_i)_j = c_j$, where $c_j$ is sampled from the uniform distribution on $[0,1)$. Since there is no restriction on how small $k$  needs to be, we set $k=75$ for CIFAR-10 and $k=40$ for MNIST and Fashion-MNIST.  We report $\delta_{\ell_{1}}, \delta_{\ell_2}, \Delta_{\ell_1}, \Delta_{l_{2}}$ and since the noise budget is in $\ell_2$-norm, we also report $\ell_{2_{\text{avg}}} := \sum_{i=1}^{500}\|e_i\|_2$. The results are shown in Table \ref{tab:l2_recovery_bp}. As was the case in the Section \ref{sub_l0}, the recovery error is well below the upper bounds of \eqref{main_result_BP_l2_eq1} and \eqref{main_result_BP_l2_eq2}.

\subsubsection{$\ell_\infty$ noise}
\label{sub_li}
We follow the same procedure as in Section \ref{sub_l2} to construct noise vectors $e_i, i=1,2, \dots,500$. This ensures that $\|e\|_\infty \leq 1$. We also select the value of $k$ as done in Section \ref{sub_l2} as well. We report $\delta_{\ell_{1}}, \delta_{\ell_2}, \Delta_{\ell_1}, \Delta_{l_{2}}$ as well as $\ell_{\infty_{\text{avg}}} := \sum_{i=1}^{500}\|e_i\|_\infty$. Once again recovery error is well below the upper bounds of \eqref{main_result_BP_li_eq1} and \eqref{main_result_BP_li_eq2} as shown in Table \ref{tab:li_recovery_ds}.

\subsection{Defense against $\ell_0$-norm attacks}
\label{l0_exp}
This section is organized as follows:  first we examine CRD against the One Pixel Attack (OPA) \cite{su2019one} for CIFAR-10. We only test the attack on CIFAR-10 as it is most effective against natural images and does not work well on MNIST or FASHION-MNIST. We note that this attack satisfies the theoretical constraints for $t$ provided  \eqref{main_result_IHT eq2} and \eqref{main_result_IHT eq3}, hence allowing us to test how well CRD works within existing guarantees. Once we establish the  effectiveness of CRD against OPA, we then test it against two other $\ell_0$-norm bounded attacks: Carlini and Wagner (CW) $\ell_0$-norm attack \cite{carlini2017towards} and the Jacobian based Saliency Map Attack (JSMA) \cite{papernot2016limitations}. Each experiment is conducted on a set of 1000 points sampled uniformly at random from the test set of the respective dataset.

\subsubsection{One Pixel Attack}
\label{opa}
We first resize all CIFAR-10 images to $125 \times 125 \times 3$ while maintaining aspect ratios to ensure that the data falls under the hypotheses of  \eqref{main_result_IHT eq2} and \eqref{main_result_IHT eq3} even for large values of $k$.  The OPA attack perturbs exactly one pixel of the image, leading to an $\ell_0$ noise budget of $t = 3$ per image. The $\ell_0$ noise budget of $t=3$ per image allows us to use $k=275$ per channel. Table \ref{tab:op_recovery} shows that OPA is very effective against natural images and forces the network to misclassify all previously correctly classified inputs. Figure \ref{fig:adv_l0_op} shows that adversarial images created using OPA are visually almost indistinguishable from the original images. We test the performance of CRD in two ways: a) reconstruction quality b) network performance on reconstructed images. 

\begin{figure*}[thb]
\begin{center}
\begin{tabular}{c c c}
Original & {\includegraphics[valign=m,width = 2.5in, height=.2in]{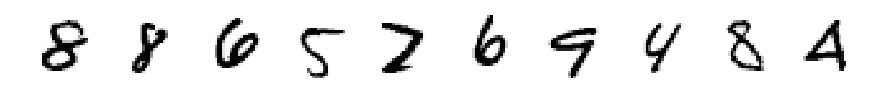}} & 
{\includegraphics[valign=m,width = 2.5in, height=.2in]{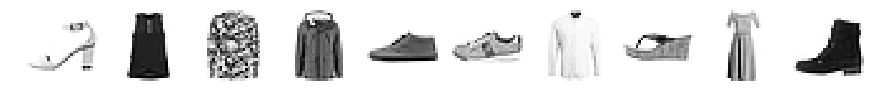}}  \\
CW $\ell_0$ &{\includegraphics[valign=m,width = 2.5in, height=.2in]{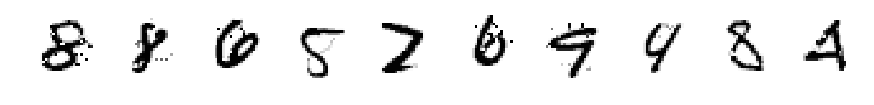}} &
{\includegraphics[valign=m,width = 2.5in]{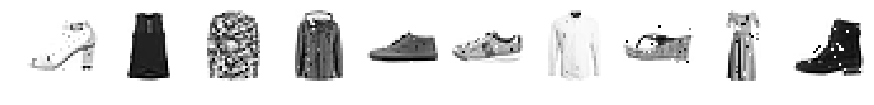}}  \\
JSMA &{\includegraphics[valign=m,width = 2.5in, height=.2in]{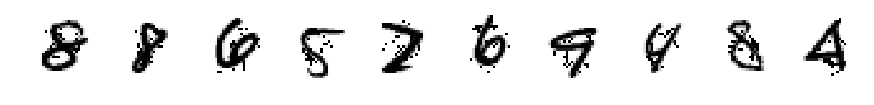}} &
{\includegraphics[valign=m,width = 2.5in, height=.2in]{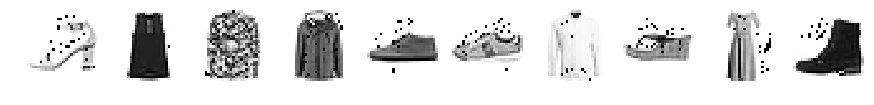}} \\
\end{tabular}
\end{center}
\caption{Adversarial images for MNIST and Fashion-MNIST datasets for $\ell_0$-norm bounded attacks. The first row lists the original images, while the second row shows  adversarial images created using the CW $\ell_0$-norm attack and the third row shows adversarial images created using the JSMA attack.}
\label{fig:adv_l0}
\end{figure*}

In order to analyse the reconstruction quality of Algorithm \ref{IHT}, we do the following: for each test image, we use  OPA to perturb the image and then use Algorithm \ref{IHT} to approximate its  largest (in absolute value) $k=275$ DCT co-efficients. We then perform the IDCT on these recovered co-efficients to generate reconstructed images. The reconstructed images from Algorithm \ref{IHT} can be seen in the second row of Figure \ref{fig:adv_l0_op_recon}. 

Noting that Algorithm \ref{IHT} leads to high quality reconstruction, we now test whether network accuracy improves on these reconstructed images. To do so, we feed these reconstructed images as input to the network and report its accuracy in Table \ref{tab:op_recovery}. We note that network performance does indeed improve as network accuracy goes from $0.0\%$ to $71.8\%$ using Algorithm \ref{IHT}. Therefore, we conclude that CRD provides a substantial improvement in accuracy in against OPA.

\subsubsection{CW-$\ell_0$ Attack and JSMA}
\label{l0_attack}
Having established the effectiveness of CRD against OPA, we move onto the CW $\ell_0$-norm attack and JSMA. We note that even when $t$ is much larger than the hypotheses of Theorem \ref{main_result_IHT} and Theorem \ref{main_result_BP_l0}, we find that Algorithms \ref{IHT} and \ref{BP} are still able to defend the network. We hypothesize that this maybe related to the behavior of the RIP of a matrix for ``most" \footnote{Recall that the results of Theorems \ref{main_result_IHT} and \ref{main_result_BP_l0} hold for all vectors in $\C^{n}$, while for the vectors we considered (CIFAR-10, MNIST, Fashion-MNIST), the recovery error in Section \ref{sub_l0} was well below the guarantees.}  vectors as opposed to the RIP for all vectors, and leave a more rigorous analysis for a follow up work.

To begin our analysis, we show adversarial images for MNIST and Fashion-MNIST created by CW-$\ell_0$ and JSMA in Figure \ref{fig:adv_l0}. The first row contains the original test images while the second and the third rows show the adversarial images. We show adversarial images for the CIFAR-10 dataset in Figure \ref{fig:adv_l0_cifar}.

\begin{figure*}[thb]
\begin{center}
\begin{tabular}{c c }
Oroginal & {\includegraphics [valign=m,width = 4.5in, height=.3in]{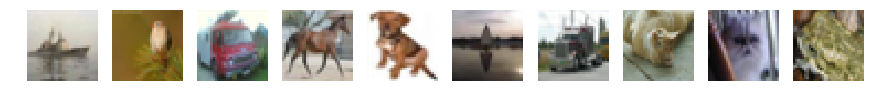}} \\
CW-$\ell_0$ & {\includegraphics[valign=m,width = 4.5in, height=.3in]{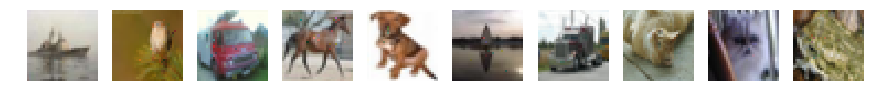} }\\
JSMA & {\includegraphics[valign=m,width = 4.5in, height=.3in]{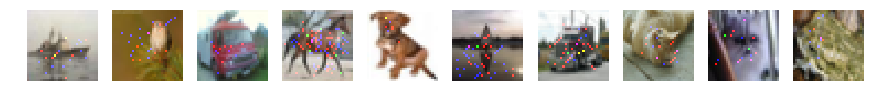}}
\end{tabular}
\end{center}
\caption{Adversarial images for the CIFAR-10 dataset showing $\ell_0$-norm attacks. The first row contains the original images, the second row shows images created using the CW $\ell_0$-norm attack and the third row shows images created using the JSMA attack.}
\label{fig:adv_l0_cifar}
\end{figure*}

\begin{figure*}[thb]
\begin{center}
\begin{tabular}{c c c}
Original & {\includegraphics[valign=m,width = 2.5in, height=.2in]{images/reg-mnist_l0.png}} & 
{\includegraphics[valign=m,width = 2.5in, height=.2in]{images/reg-fmnist_l0.png}}  \\
CW $\ell_0$ Rec. &{\includegraphics[valign=m,width = 2.5in, height=.2in]{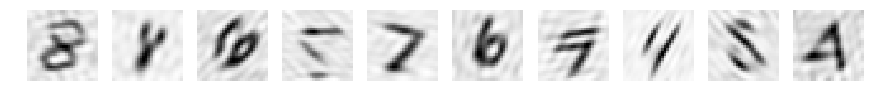}} &
{\includegraphics[valign=m,width = 2.5in]{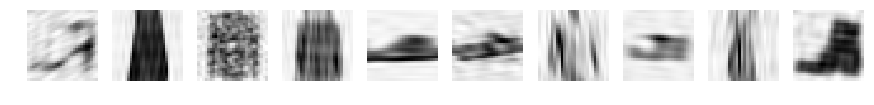}}  \\
JSMA Rec. &{\includegraphics[valign=m,width = 2.5in, height=.2in]{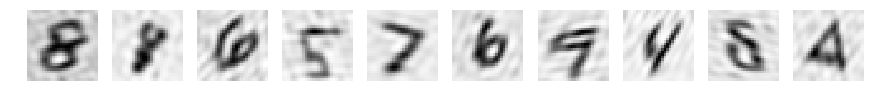}} &
{\includegraphics[valign=m,width = 2.5in, height=.2in]{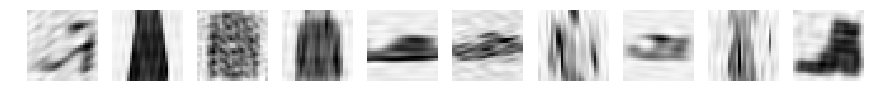}} \\
\end{tabular}
\end{center}
\caption{Reconstruction from adversarial images using Algorithm \ref{IHT}. The first row shows the original images while the second and third rows show the reconstruction of the adversarial images after recovering the largest $40$ co-efficients using Algorithm \ref{IHT}.}
\label{fig:adv_l0_recon}
\end{figure*}

\begin{figure*}[thb]
\begin{center}
\begin{tabular}{c c c}
Original & {\includegraphics[valign=m,width = 2.5in, height=.2in]{images/reg-mnist_l0.png}} & 
{\includegraphics[valign=m,width = 2.5in, height=.2in]{images/reg-fmnist_l0.png}}  \\
CW $\ell_0$ Rec. &{\includegraphics[valign=m,width = 2.5in, height=.2in]{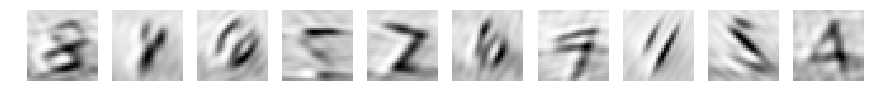}} &
{\includegraphics[valign=m,width = 2.5in]{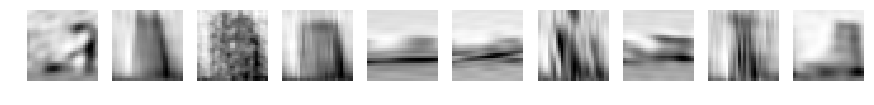}}  \\
JSMA Rec. &{\includegraphics[valign=m,width = 2.5in, height=.2in]{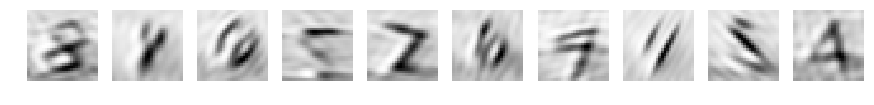}} &
{\includegraphics[valign=m,width = 2.5in, height=.2in]{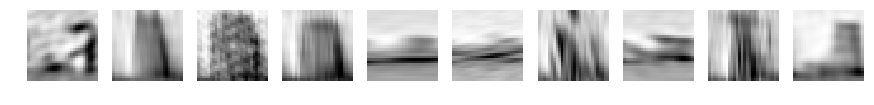}} \\
\end{tabular}
\end{center}
\caption{Reconstruction quality of Algorithm \ref{BP}. The first row shows the original images while the second and third rows show the reconstruction of the adversarial images after recovering the largest $40$ co-efficients using Algorithm \ref{BP}.}
\label{fig:adv_l0_recon_bp_mnist}
\end{figure*}
\begin{figure*}[thb]
\begin{center}
\begin{tabular}{c c}
Original & {\includegraphics [valign=m,width = 4.5in, height=.3in]{images/reg-cifar10_l0.png}} \\
CW $\ell_0$ Rec. & {\includegraphics[valign=m,width = 4.5in, height=.3in]{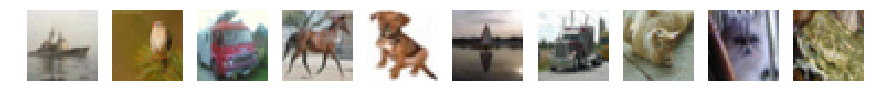} }\\
JSMA Rec.& {\includegraphics[valign=m,width = 4.5in, height=.3in]{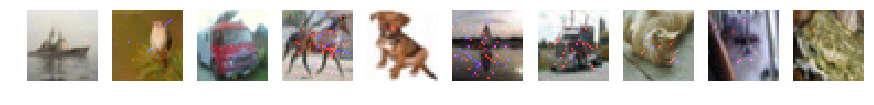} }
\end{tabular}
\end{center}
\caption{Reconstruction quality of Algorithm \ref{IHT} for CIFAR-10. The first row shows the original images while the second and the third rows show reconstructions of the CW-$\ell_0$ and JSMA adversarial images after recovering the largest $500$ co-efficients via Algorithm \ref{IHT} }
\label{fig:adv_l0_recon_iht_cifar}
\end{figure*}

\begin{figure*}[thb]
\begin{center}
\begin{tabular}{c c}
Original & {\includegraphics [valign=m,width = 4.5in, height=.3in]{images/reg-cifar10_l0.png}} \\
CW $\ell_0$ Rec. & {\includegraphics[valign=m,width = 4.5in, height=.3in]{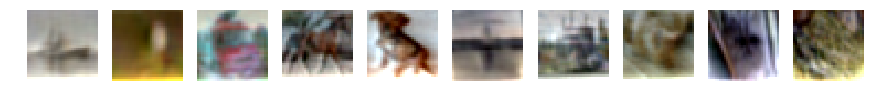} }\\
JSMA Rec. & {\includegraphics[valign=m,width = 4.5in, height=.3in]{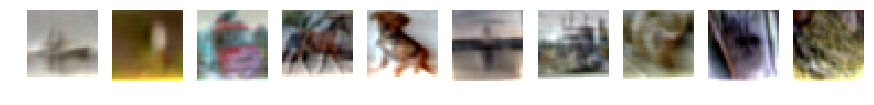} }
\end{tabular}
\end{center}
\caption{Reconstruction quality of Algorithm \ref{BP} on CIFAR-10. The first row shows the  original images while the second and the third rows show reconstructions of the CW-$\ell_0$ and JSMA adversarial images after recovering all 1024 co-efficients using Algorithm \ref{BP}. }
\label{fig:adv_l0_recon_bp_cifar}
\end{figure*}

\begin{table*}[t]
\begin{center}
\begin{tabular}{|c|c|c|c|c|c|c|c|c|c|}
\hline
{\bf Dataset} & {\bf Orig.}  & \multicolumn{4}{|c|}{\bf C\&W $\ell_0$} &  \multicolumn{4}{|c|}{\bf JSMA} \\
&Acc.& $t_{\text{avg}}$ & Acc. & IHT Acc. & BP Acc. & $t_{\text{avg}}$ & Acc. & IHT Acc. & BP Acc.\\
\hline
CIFAR-10 &  84.9\%  & 18 & 8.7\% & 83.0\% &67.0\% &34 & 2.7\% & 63.2\% &67.3\% \\
\hline
MNIST &  98.8\%  & 15 & 0.9\% & 84.2\%& 55.9\% &17 & 56.5 \% & 90.1\%&67.4\% \\
\hline
F-MNIST   & 91.8\%  & 16 & 5.27\% & 84.1\%&71.4\% & 17 & 62.6 \% & 83.3\%&72.0\% \\
\hline
\end{tabular}
\end{center}
\caption{Network performance on the original inputs, adversarial inputs and the inputs corrected using CRD. Here the $t_{\text{avg}}$ column lists the average adversarial budget for each attack, Orig. Acc. column lists the accuracy of the network on the original inputs, the Acc. columns shows the accuracy on adversarial inputs, the IHT Acc. and the BP Acc. columns list the accuracy of the network on inputs that have been corrected using Algorithm \ref{IHT} and Algorithm \ref{BP} respectively.}
\label{tab:l0_norm_table}
\end{table*}
\begin{figure*}[thb]
\begin{center}
\begin{tabular}{c c}
Original & {\includegraphics [valign=m,width = 4.5in, height=.3in]{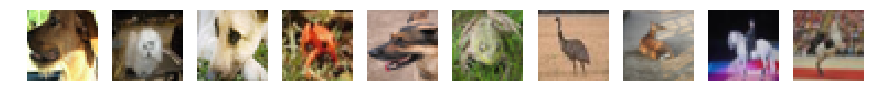}} \\
CW $\ell_2$ & {\includegraphics[valign=m,width = 4.5in, height=.3in]{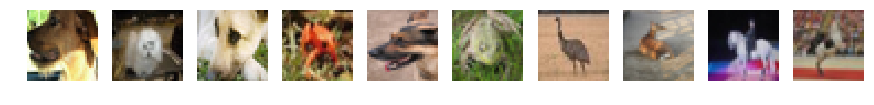} }\\
DF & {\includegraphics[valign=m,width = 4.5in, height=.3in]{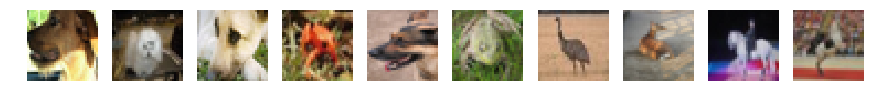} }
\end{tabular}
\end{center}
\caption{Adversarial images for the CIFAR-10 dataset showing $\ell_2$-norm attacks. The first row contains the original images, the second row shows images created using the CW $\ell_2$-norm attack and the third row shows images created using the DeepFool attack. }
\label{fig:adv_l2_cifar}
\end{figure*}
Next, we  follow the procedure described in Section \ref{opa} to analyze the quality of reconstructions for Algorithm \ref{IHT} and Algorithm \ref{BP}. For MNIST and Fashion-MNIST, we show the reconstructions of Algorithm \ref{IHT} in Figure \ref{fig:adv_l0_recon} and for Algorithm \ref{BP} in Figure \ref{fig:adv_l0_recon_bp_mnist}. For CIFAR-10, we show the reconstructions for Algorithm \ref{IHT} in Figure \ref{fig:adv_l0_recon_iht_cifar} and for Algorithm \ref{BP} in Figure \ref{fig:adv_l0_recon_bp_cifar}. In each case it can be seen that both algorithms provide high quality reconstructions for values of $t$ that are well outside the hypotheses required by Theorem \ref{main_result_IHT} and Theorem \ref{main_result_BP_l0}.  We report these $t$ values and the improvement in network performance on reconstructed adversarial images using CRD in Table \ref{tab:l0_norm_table}.

\begin{table*}[t]
\begin{center}
\begin{tabular}{|c|c|c|c|c|c|c|c|}
\hline
{\bf Dataset} & {\bf Orig.}  & \multicolumn{3}{|c|}{\bf C\&W $\ell_2$} &  \multicolumn{3}{|c|}{\bf Deepfool} \\
&Acc.&  $\ell_{2_{\text{avg}}}$ & Acc. & BP Acc. & $\ell_{2_{\text{avg}}}$ & Acc. & BP Acc. \\
\hline
CIFAR-10 &  84.9\%  & 0.12 & 8.7\% & 72.3\% &0.11 & 7.7\% & 71.6\% \\
\hline
MNIST &  99.17\%  & 1.35 & 0.9\% & 92.4\% &1.72 &1.1 \% & 90.7\% \\
\hline
Fashion-MNIST   &  90.3\%  & 0.61& 5.4\% & 78.3\% & 0.63& 5.5 \% & 76.4\% \\
\hline
\end{tabular}
\end{center}
\caption{Accuracy of our network on the original inputs, adversarial inputs and the inputs corrected using CRD. Here the $\ell_{2_\text{avg}}$ column lists the average $\ell_2$-norm of the attack vector, Acc. columns list the accuracy of the network on the original and adversarial inputs, and the BP Acc. columns lists the accuracy of the network on inputs reconstructed using Algorithm \ref{BP}.}
\label{tab:l2_norm_table}
\end{table*}

\subsection{Defense against $\ell_2$-norm attacks}
\label{l2_exp}
In the case of $\ell_2$-norm bounded attacks, we use the CW $\ell_2$-norm attack \cite{carlini2017towards} and the Deepfool attack \cite{moosavi2016deepfool} as they have been shown to be the most powerful. We note that Theorem \ref{main_result_BP_l2} does not impose any restrictions on $k$ or $t$ and therefore the guarantees of equations \eqref{main_result_BP_l2_eq1} and \eqref{main_result_BP_l2_eq2} are applicable for recovery in all experiments of this section. Figure \ref{fig:adv_l2_cifar} shows examples of each attack for the CIFAR-10 dataset while adversarial images for MNIST and Fashion-MNIST are presented in Figure \ref{fig:adv_l2}.

\begin{table*}[t]
\begin{center}
\begin{tabular}{|c|c|c|c|c|}
\hline
{\bf Dataset} & {\bf Orig.}  & \multicolumn{3}{|c|}{\bf BIM} \\
&Acc.&  $\ell_{\infty_{\text{avg}}}$ & Acc. & DS Acc.  \\
\hline
CIFAR-10 &  84.9\%  & 0.015 & 7.4\% & 49.4\%  \\
\hline
MNIST &  99.17\%  & 0.15 & 4.9\% &74.7\% \\
\hline
Fashion-MNIST   &  90.3\%  & 0.15 & 5.3\% & 57.5\%  \\
\hline
\end{tabular}
\end{center}
\caption{Defense against $\ell_\infty$ attacks using CRD. Here $\ell_{\infty_\text{avg}}$ column lists the  $\ell_\infty$-norm of each attack vector, Acc. columns list the accuracy of the network on the original and adversarial inputs, and the DS Acc. columns lists the accuracy of the network on inputs reconstructed using Algorithm \ref{DS}.}
\label{tab:li_norm_table}
\end{table*}

\begin{figure*}[thb]
\begin{tabular}{c c c}
Original & {\includegraphics[valign=m,width = 2.5in, height=.2in]{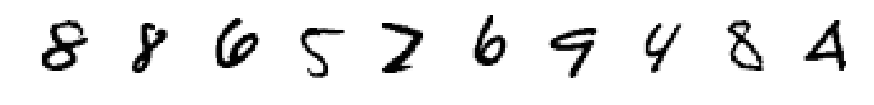}} & 
{\includegraphics[valign=m,width = 2.5in, height=.2in]{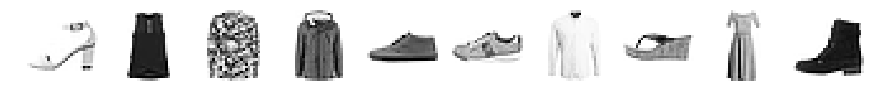}}  \\
CW $\ell_2$ &{\includegraphics[valign=m,width = 2.5in, height=.2in]{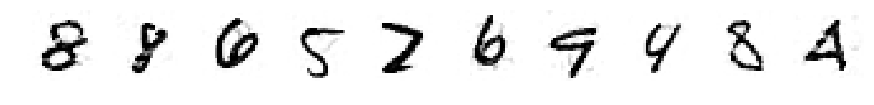}} &
{\includegraphics[valign=m,width = 2.5in]{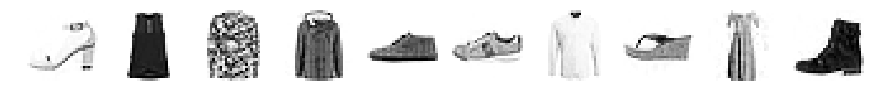}}  \\
Deepfool &{\includegraphics[valign=m,width = 2.5in, height=.2in]{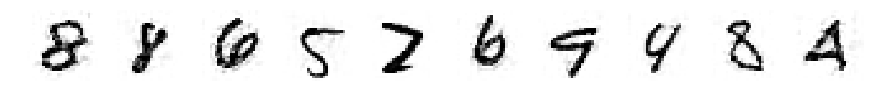}} &
{\includegraphics[valign=m,width = 2.5in, height=.2in]{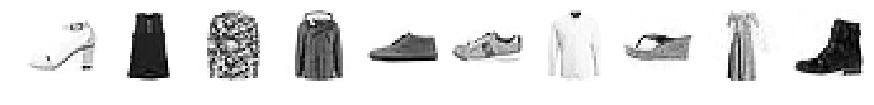}} \\

\end{tabular}
\caption{Adversarial images for MNIST and Fashion-MNIST datasets for $\ell_2$-norm bounded attacks. The first row lists the original images for the MNIST and Fashion MNIST dataset. The second row shows  adversarial images created using the CW $\ell_2$-norm attack and the third row shows adversarial images created using the Deepfool attack.}
\label{fig:adv_l2}
\end{figure*}

\begin{figure*}[thb]
\begin{center}
\begin{tabular}{c c c}
Original & {\includegraphics[valign=m,width = 2.5in, height=.2in]{images/reg-mnist_l2.png}} & 
{\includegraphics[valign=m,width = 2.5in, height=.2in]{images/reg-fmnist_l2.png}}  \\
CW $\ell_2$ &{\includegraphics[valign=m,width = 2.5in, height=.2in]{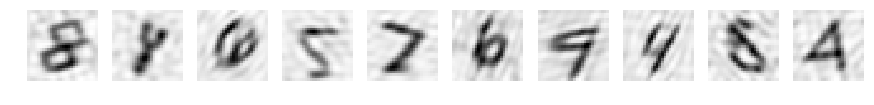}} &
{\includegraphics[valign=m,width = 2.5in]{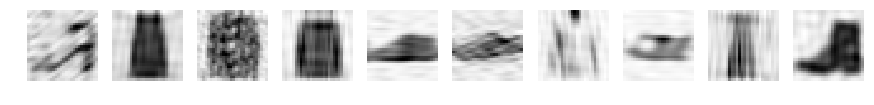}}  \\
Deepfool &{\includegraphics[valign=m,width = 2.5in, height=.2in]{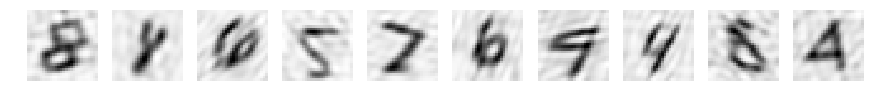}} &
{\includegraphics[valign=m,width = 2.5in, height=.2in]{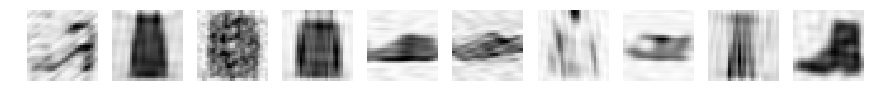}} \\
\end{tabular}
\end{center}
\caption{Reconstruction from adversarial images using Algorithm \ref{BP}. The first row shows the original images while the second and the third rows show the reconstruction of the adversarial images after recovering the largest $40$ co-efficients using Algorithm \ref{BP}.}
\label{fig:adv_l2_recon}
\end{figure*}

The reconstruction quality for MNIST and Fashion-MNIST is shown in Figure \ref{fig:adv_l2_recon} and for CIFAR-10 we show the reconstruction quality in Figure \ref{fig:adv_l2_cifar_recon_bp}.  It can be noted that reconstruction using Algorithm \ref{BP} is of high quality for all three datasets. In order to check whether this high quality reconstruction also leads to improved performance in network accuracy, we test each network on reconstructed images using Algorithm \ref{BP}. We report the results in Table \ref{tab:l2_norm_table} and note that Algorithm \ref{BP} provides a substantial improvement in network accuracy for each dataset and each attack method used.

\begin{figure*}[thb]
\begin{center}
\begin{tabular}{c c}
Original & {\includegraphics [valign=m,width = 4.5in, height=.3in]{images/reg-cifar10_l2.png}} \\
CW $\ell_2$ Rec. & {\includegraphics[valign=m,width = 4.5in, height=.3in]{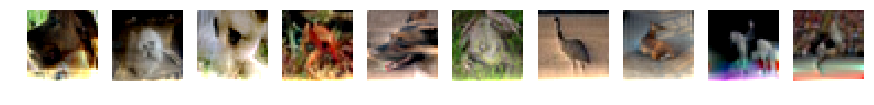} }\\
DF Rec. & {\includegraphics[valign=m,width = 4.5in, height=.3in]{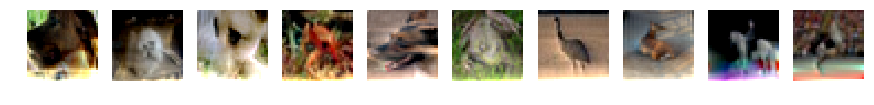} }
\end{tabular}
\end{center}
\caption{Reconstruction quality of Algorithm \ref{BP}. The first row shows the original images, the second row shows reconstruction of CW-$\ell_2$ and DeepFool images using $1024$ co-efficients recovered via Algorithm \ref{BP}. }
\label{fig:adv_l2_cifar_recon_bp}
\end{figure*}

\subsection{Defense against $\ell_\infty$-norm attacks}
\label{li_exp}
For $\ell_\infty$-norm bounded attacks, we use the BIM attack \cite{kurakin2016adversarial} as it is has been shown to be very effective and also allows one to control the $\ell_\infty$-norm of the attack vector explicitly. We note that while the CW $\ell_\infty$-norm attack \cite{carlini2017towards} has the ability to create attack vectors with $\ell_\infty$-norm less than or equal to BIM, it is computationally expensive and also does not allow one to pre-specify a value for the $\ell_\infty$-norm of an attack vector. Therefore, we limit our experimental analysis to the BIM attack.  Note that for any attack vector $e$,  $\|e\|_2 \leq \sqrt{n}\|e\|_\infty$ hence allowing $\ell_\infty$-norm attacks to create attack vectors with large $\ell_2$-norm.  Therefore, we could expect reconstruction quality and network accuracy to be lower when compared to $\ell_2$-norm attacks. Figure \ref{fig:adv_li} shows examples of each attack for the  MNIST and Fashion-MNIST datasets while images for CIFAR-10 are presented in Figure \ref{fig:adv_li_cifar}. We show reconstructions using Algorithm \ref{DS} in Figures \ref{fig:adv_li_recon} and \ref{fig:adv_li_cifar_recon_ds}. Finally, we report the network performance on reconstructed inputs using Algorithm \ref{DS} in Table \ref{tab:li_norm_table}. We note that Algorithm \ref{DS} provides an increase in network performance against reconstructed adversarial inputs. However, the improvement in performance is not as substantial as it was against $\ell_0$ or $\ell_2$-norm attacks. 

\begin{figure*}[thb]
\begin{center}
\begin{tabular}{c c c}
Original & {\includegraphics[valign=m,width = 2.5in, height=.2in]{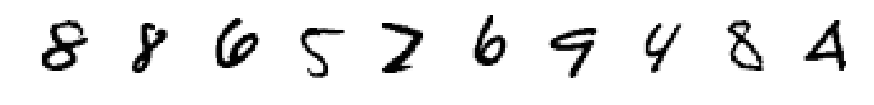}} & 
{\includegraphics[valign=m,width = 2.5in, height=.2in]{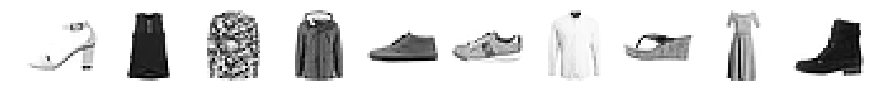}}  \\
BIM &{\includegraphics[valign=m,width = 2.5in, height=.2in]{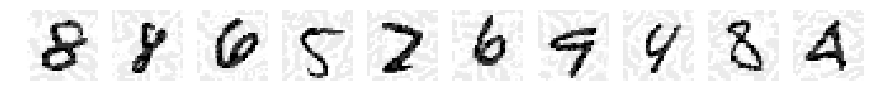}} &
{\includegraphics[valign=m,width = 2.5in]{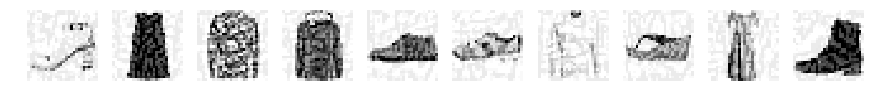}}  \\
\end{tabular}
\end{center}
\caption{Adversarial images for MNIST and Fashion-MNIST datasets for $\ell_\infty$-norm bounded attacks. The first row lists the original images, while the second row shows  adversarial images created using the BIM attack.}
\label{fig:adv_li}
\end{figure*}

\begin{figure*}[thb]
\begin{center}
\begin{tabular}{c c}
Original & {\includegraphics [valign=m,width = 4.5in, height=.3in]{images/reg-cifar10_l2.png}} \\
BIM & {\includegraphics[valign=m,width = 4.5in, height=.3in]{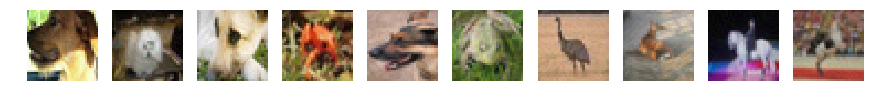} }
\end{tabular}
\end{center}
\caption{Adversarial images for the CIFAR-10 dataset showing the BIM attack. The first row contains the original images, the second row shows images created using the BIM attack.}
\label{fig:adv_li_cifar}
\end{figure*}

\begin{figure*}[thb]
\begin{center}
\begin{tabular}{c c c}
Original & {\includegraphics[valign=m,width = 2.5in, height=.2in]{images/reg-mnist_li.png}} & 
{\includegraphics[valign=m,width = 2.5in, height=.2in]{images/reg-fmnist_li.png}}  \\
BIM &{\includegraphics[valign=m,width = 2.5in, height=.2in]{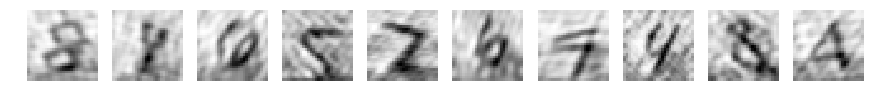}} &
{\includegraphics[valign=m,width = 2.5in]{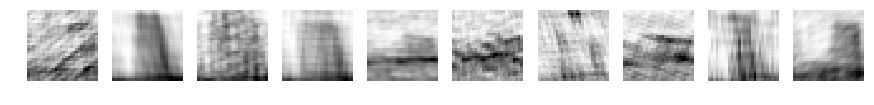}}  \\
\end{tabular}
\end{center}
\caption{Reconstruction from adversarial images using Algorithm \ref{DS}. The first row shows the original images while the second and the third rows show the reconstruction of the adversarial images after recovering the largest $40$ co-efficients using Algorithm \ref{DS}.}
\label{fig:adv_li_recon}
\end{figure*}
\begin{figure*}[thb]
\begin{center}
\begin{tabular}{c c}
Original & {\includegraphics [valign=m,width = 4.5in, height=.3in]{images/reg-cifar10_l2.png}} \\
BIM Rec. & {\includegraphics[valign=m,width = 4.5in, height=.3in]{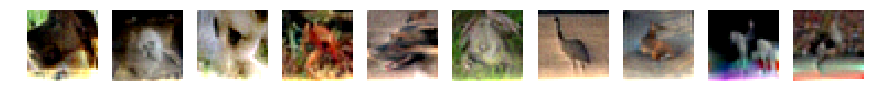} }
\end{tabular}
\end{center}
\caption{Reconstruction quality of Algorithm \ref{DS}. The first row shows the original images, the second row shows reconstruction of BIM images using $1024$ co-efficients recovered via Algorithm \ref{DS}. }
\label{fig:adv_li_cifar_recon_ds}
\end{figure*}

\subsection{Which recovery algorithm to use for $\ell_0$-norm attacks}
As shown in Section \ref{sub_l0}, Algorithm \ref{IHT} and  Algorithm \ref{BP} lead to high quality reconstructions for $\ell_0$-norm bounded attacks. Hence, it is conceivable that CRD using either algorithm should be able to provide a good defense. However, we note that the $\ell_2$-norm recovery error is lower for Algorithm \ref{IHT} as seen in Section \ref{sub_l0}. Therefore, depending on the dataset, Algorithm \ref{IHT}  may lead to better quality reconstructions and hence better network accuracy. From a practical perspective, one may ask which algorithm is faster. Since Algorithm \ref{BP} is not technically an algorithm, its runtime is dependent on the actual method used to solve the optimization problem. For instance, we use Second Order Cone Programming (SOCP) from CVXPY \cite{cvxpy} for solving the minimization problem in Algorithm \ref{BP}. In our experiments, we noticed that the runtime of Algorithm \ref{BP} slows considerably for larger values of $n$. However, Algorithm \ref{IHT} does not face this issue (there is a slowdown but it is much smaller than Algorithm \ref{BP}). Therefore, if speed is important, it may be beneficial to use Algorithm \ref{IHT} as opposed to Algorithm \ref{BP}.
\clearpage

\section{Conclusion}
\label{conclusion}
We provided recovery guarantees for corrupted signals in the case of $\ell_0$-norm,  $\ell_2$-norm, and $\ell_\infty$-norm bounded  noise. We then experimentally verified these guarantees and showed that for the datasets used, recovery error was considerably lower than the upper bounds of our theorems. We were able to utilize these observations in CRD and improve the performance of neural networks substantially in the case of $\ell_0$-norm, $\ell_2$-norm and $\ell_\infty$-norm bounded noise. While $\ell_0$-norm attacks don't necessarily satisfy the constraints required by Theorem \ref{main_result_IHT} and Theorem \ref{main_result_BP_l0}, we showed that CRD is still able to provide a good defense for values of $t$ much larger than allowed in the guarantees. The guarantees of  Theorem \ref{main_result_BP_l2} and Theorem \ref{main_result_BP_li} were applicable in all experiments and CRD was shown to improve network performance for all attacks.

\newpage

\bibliographystyle{plain}
\clearpage
\bibliography{bibliography}

\end{document}